\documentclass[twoside]{article}

\usepackage[accepted]{aistats2023}
\usepackage[round]{natbib}

\usepackage[utf8]{inputenc} 
\usepackage[T1]{fontenc}    
\usepackage{hyperref}       
\usepackage{url}            
\usepackage{booktabs}       
\usepackage{enumerate}

\usepackage{graphicx}
\usepackage{caption}
\usepackage{subcaption}
\usepackage[table]{xcolor}
\usepackage{amssymb}
\usepackage{amsmath}
\usepackage{amsthm}
\usepackage{tablefootnote}

\usepackage{dsfont}

\newcommand{\ind}{\mathds{1}}

\newcommand{\R}{\mathds{R}}
\newcommand{\cD}{\mathcal{D}}
\newcommand{\cmark}{\checkmark}%
\newcommand{\xmark}{$\times$}%

\newtheorem{theorem}{Theorem}

\begin{document}

%

%

\twocolumn[

\aistatstitle{Membership Inference Attacks against Synthetic Data through Overfitting Detection}

\aistatsauthor{ Boris van Breugel \And Hao Sun \And Zhaozhi Qian \And  Mihaela van der Schaar }

\aistatsaddress{ University of Cambridge \And University of Cambridge \And University of Cambridge \And University of Cambridge \\Alan Turing Institute } ]

\begin{abstract}
Data is the foundation of most science. Unfortunately, sharing data can be obstructed by the risk of violating data privacy, impeding research in fields like healthcare. Synthetic data is a potential solution. It aims to generate data that has the same distribution as the original data, but that does not disclose information about individuals. Membership Inference Attacks (MIAs) are a common privacy attack, in which the attacker attempts to determine whether a particular real sample was used for training of the model. Previous works that propose MIAs against generative models either display low performance---giving the false impression that data is highly private---or need to assume access to internal generative model parameters---a relatively low-risk scenario, as the data publisher often only releases synthetic data, not the model. In this work we argue for a realistic MIA setting that assumes the attacker has some knowledge of the underlying data distribution. We propose DOMIAS, a density-based MIA model that aims to infer membership by targeting local overfitting of the generative model. Experimentally we show that DOMIAS is significantly more successful at MIA than previous work, especially at attacking uncommon samples. The latter is disconcerting since these samples may correspond to underrepresented groups. We also demonstrate how DOMIAS' MIA performance score provides an interpretable metric for privacy, giving data publishers a new tool for achieving the desired privacy-utility trade-off in their synthetic data.
\end{abstract}

\section{INTRODUCTION} 
Real data may be privacy-sensitive, prohibiting open sharing of data and in turn hindering new scientific research, reproducibility, and the development of machine learning itself. Recent advances in generative modelling provide a promising solution, by replacing the \textit{real} dataset with a \textit{synthetic} dataset---which retains most of the distributional information, but does not violate privacy requirements. 

\textbf{Motivation} The motivation behind synthetic data is that data is generated \emph{from scratch}, such that no synthetic sample can be linked back to any single real sample. However, how do we verify that samples indeed cannot be traced back to a single individual? Some generative methods have been shown to memorise samples during the training procedure, which means the synthetic data samples---which are thought to be genuine---may actually reveal highly private information \citep{Carlini2018TheNetworks}. To mitigate this, we require good metrics for evaluating privacy, and this is currently one of the major challenges in synthetic data \citep{Jordon2021Hide-and-SeekRe-identification, Alaa2022HowModels}. Differential privacy (DP) \citep{Dwork2014ThePrivacy} is a popular privacy definition and used in several generative modelling works \citep{Ho2021DP-GAN:Nets,Torkzadehmahani2020DP-CGAN:Generation,Chen2020GS-WGAN:Generators,Jordon2019PATE-GAN:Guarantees,Long2019G-PATE:Discriminators,Wang2021DataLens:Aggregation,Cao2021DontDivergence}. However, even though DP is theoretically sound, its guarantees are difficult to interpret and many works \citep{Rahman2018MembershipModel,Jayaraman2019EvaluatingPractice,Jordon2019PATE-GAN:Guarantees,Ho2021DP-GAN:Nets} reveal that for many settings, either the theoretical privacy constraint becomes meaningless ($\epsilon$ becomes too big), or utility is severely impacted. This has motivated more lenient privacy definitions for synthetic data, e.g. see \citep{Yoon2020Anonymizationads-gan}. We take an adversarial approach by developing a privacy attacker model---usable as synthetic data evaluation metric that quantifies the practical privacy risk. 

\textbf{Aim} Developing and understanding privacy attacks against generative models are essential steps in creating better private synthetic data. There exist different privacy attacks in machine learning literature---see e.g. \citep{Rigaki2020ALearning}---but in this work we focus on Membership Inference Attacks (MIAs) \citep{Shokri2017MembershipModels}. The general idea is that the attacker aims to determine whether a particular sample they possess was used for training the machine learning model. Successful MIA poses a privacy breach, since mere membership to a dataset can be highly informative. For example, an insurance company may possess a local hospital's synthetic cancer dataset, and be interested to know whether some applicant was used for generating this dataset---disclosing that this person likely has cancer \citep{Hu2022MembershipSurvey}. Additionally, MIAs  can be a first step towards other privacy breaches, like profiling or property inference \citep{DeCristofaro2021ALearning}. 

Previous work in MIA attacks against generative models is inadequate, conveying a false pretense of privacy. In the NeurIPS 2020 Synthetic Data competition \citep{Jordon2021Hide-and-SeekRe-identification}, none of the attackers were successful at MIA.\footnote{Specifically, none performed better than random guessing in at least half of the datasets.} Similar negative results were found in the black-box results of \citep{Liu2019PerformingModels,Hayes2019LOGAN:Models,Hilprecht2019MonteModels, Chen2019GAN-Leaks:Models}, where additional assumptions were explored to create more successful MIAs. Most of these assumptions (see Sec. \ref{sec:related}) rely on some access to the generator, which we deem relatively risk-less since direct access is often avoidable in practice. Nonetheless, we show that even in the black-box setting---in which we only have access to the synthetic data---MIA can be significantly more successful than appears in previous work, when we assume the attacker has some independent data from the underlying distribution. In Sec. \ref{sec:MIA_formalism} we elaborate further on why this is a realistic assumption. Notably, it also allows an attacker to perform significantly better attacks against underrepresented groups in the population (Sec. \ref{sec:underrepresented}).

\textbf{Contributions} This paper's main contributions are the following. 
\begin{enumerate}
    \item We propose DOMIAS: a membership inference attacker model against synthetic data, that incorporates density estimation to detect generative model overfitting. DOMIAS improves upon prior MIA work by i) leveraging access to an independent reference dataset and ii) incorporating recent advances in deep density estimation. 
    \item We compare the MIA vulnerability of a range of generative models, showcasing how DOMIAS can be used as a metric that enables generative model design choices 
    \item We find that DOMIAS is more successful than previous MIA works at attacking underrepresented groups in synthetic data. This is disconcerting and strongly motivates further research into the privacy protection of these groups when generating synthetic data. 
\end{enumerate}


\section{MEMBERSHIP INFERENCE: FORMALISM AND ASSUMPTIONS} \label{sec:MIA_formalism}
\textbf{Formalism for synthetic data MIA}
Membership inference aims to determine whether a given sample comes from the training data of some model \citep{Shokri2017MembershipModels}. Let us formalise this for the generative setting. Let random variable $X$ be defined on $\mathcal{X}$, with distribution $p_R(X)$. Let $\cD_{mem}\overset{iid}{\sim} p_R(X)$ be a training set of independently sampled points from distribution $p_R(X)$. Now let $G:\mathcal{Z}\rightarrow \mathcal{X}$ be a generator that generates data given some random (e.g. Gaussian) noise $Z$. Generator $G$ is trained on $\cD_{mem}$, and is subsequently used to generate synthetic dataset $\cD_{syn}$. Finally, let $A:\mathcal{X}\rightarrow [0,1]$ be the attacker model, that possesses the synthetic dataset $\cD_{syn}$, some test point $x^*$, with $X^*\sim p_R(X)$, and possibly other knowledge---see below. Attacker $A$ aims to determine whether some $x^*\sim p_R(X)$ they possess, belonged to $\cD_{mem}$, hence the perfect attacker outputs $A(x^*)=\ind[x^*\in\cD_{mem}]$. The MIA performance of an attacker can be measured using any classification metric. 

\textbf{Assumptions on attacker access} The strictest black-box MI setting assumes the attacker only has access to the synthetic dataset $\cD_{syn}$ and test point $x^*$. In this work we assume access to a real data set that is independently sampled from $p_R(X)$, which we will call the reference dataset and denote by $\cD_{ref}$. The main motivation of this assumption is that an attacker needs some understanding of what real data looks like to infer MI---in Sec. \ref{sec:method} we will elaborate further on this assumption's benefits. Similar assumptions have been made in the supervised learning MI literature, see e.g. \citep{Shokri2017MembershipModels, Ye2021EnhancedModels}.
This is a realistic scenario to consider for data publishers: though they can control the sharing of their own data, they cannot control whether attackers acquires similar data from the general population. A cautious data publisher would assume the attacker has access to a sufficiently large $\cD_{ref}$ to approximate $p_R(X)$ accurately, since this informally bounds the MIA risk from above. Related MI works \citep{Liu2019PerformingModels,Hayes2019LOGAN:Models,Hilprecht2019MonteModels, Chen2019GAN-Leaks:Models} consider other assumptions that all require access to the synthetic data's generative model.\footnote{Though with varying extents, see \citep{Chen2019GAN-Leaks:Models}} These settings are much less dangerous to the data publisher, since these can be avoided by only publishing the synthetic data. Individual assumptions of related works are discussed further in Sec. \ref{sec:related}.


\begin{figure*}[t]
    \centering
    \begin{subfigure}[b]{0.48\textwidth}
    \centering
    \includegraphics[width=0.9\textwidth]{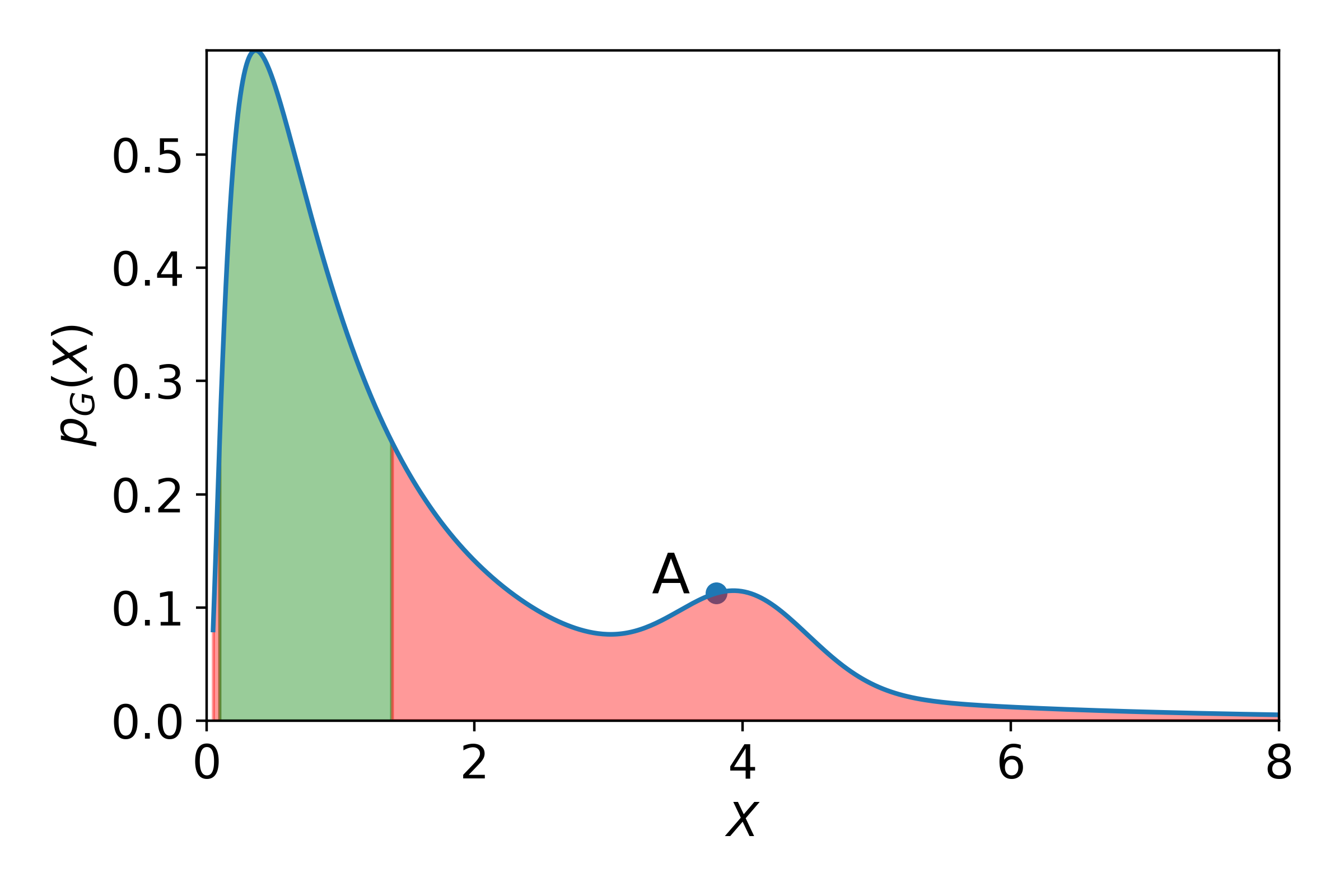}
    \caption{Generative distribution in original space}
    \end{subfigure}
    \hfill
    \begin{subfigure}[b]{0.48\textwidth}
    \centering
    \includegraphics[width=0.9\textwidth]{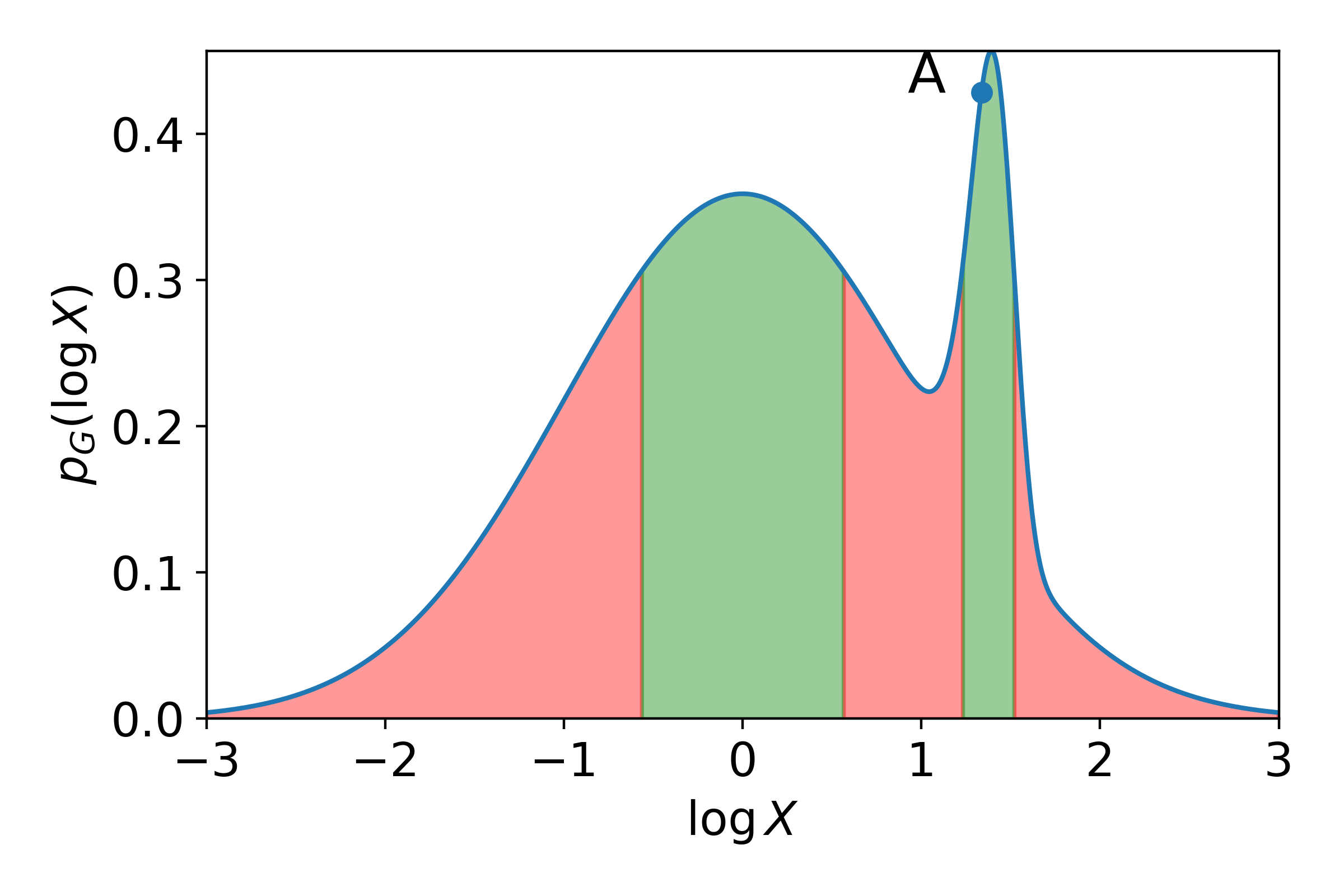}
    \caption{Distribution in log-transformed space}
    \end{subfigure}
    \caption{Should we infer membership $m=1$ for point $A$? Consider the generative distribution for two representations of $X$, optimal methods based on Eq. \ref{eq:assumption_prev} will infer $m=1$ for green and $m=0$ for red areas. This is problematic; it implies inference of these methods is dependent on the (possibly arbitrary) representation of variable $X$. \emph{Conclusion: it does not make sense to focus on mere density, MIA needs to target local overfitting directly}. This requires data from (or assumptions on) the underlying distribution.}
    \label{fig:toy_example_eq1}
\end{figure*}

\section{DOMIAS} \label{sec:method}
\subsection{Rethinking the black-box setting: why $\cD_{syn}$ alone is insufficient} 
The most popular black-box setting assumes only access to $\cD_{syn}$. This gives little information, which is why previous black-box works \citep{Hayes2019LOGAN:Models, Hilprecht2019MonteModels, Chen2019GAN-Leaks:Models} implicitly assume: 
\begin{equation}
\label{eq:assumption_prev}
    A_{prev}(x^*) = f(p_G(x^*)),
\end{equation}
where $A$ indicates the attacker's MIA scoring function, $p_G(\cdot)$ indicates the generator's output distribution and $f:\R\rightarrow [0,1]$ is some monotonically increasing function.
There are two reasons why Eq. \ref{eq:assumption_prev} is insufficient. First, the score does not account for the intrinsic distribution of the data. Consider the toy example in Figure \ref{fig:toy_example_eq2}a. There is a local density peak at $x=4$, but without further knowledge we cannot determine whether this corresponds to an overfitted example or a genuine peak in the real distribution. \textbf{It is thus naive to think we can do MI without background knowledge}. 

Second, the RHS of Eq. \ref{eq:assumption_prev} is not invariant w.r.t. bijective transformations of the domain. Consider the left and right plot in Figure \ref{fig:toy_example_eq1}. Given the original representation, we would infer $M=0$ for any point around $x=4$, whereas in the right plot we would infer $M=1$ for the same points. This dependence on the representation is highly undesirable, as any invertible transformation of the representation should contain the same information. 

How do we fix this? We create the following two desiderata: i) the MI score should target overfitting \textit{w.r.t. the real distribution}, and ii) it should be independent of representation.

%

\begin{figure*}[t]
    \centering
    \begin{subfigure}{0.48\textwidth}
    \centering
    \includegraphics[width=0.9\textwidth]{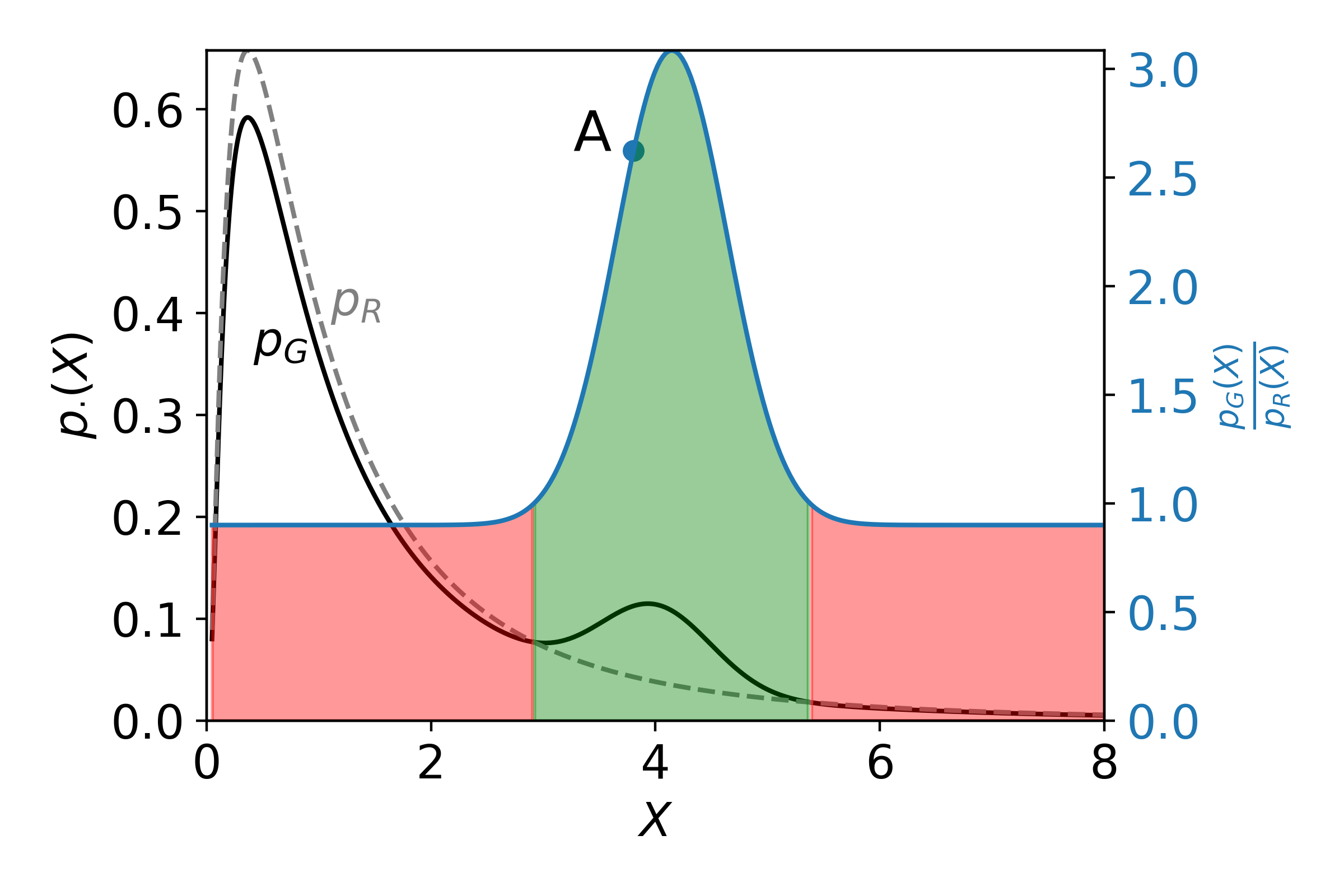}
    
    \caption{Original space}
    \end{subfigure}
    \hfill
    \begin{subfigure}{0.48\textwidth}
    \centering
    \includegraphics[width=0.9\textwidth]{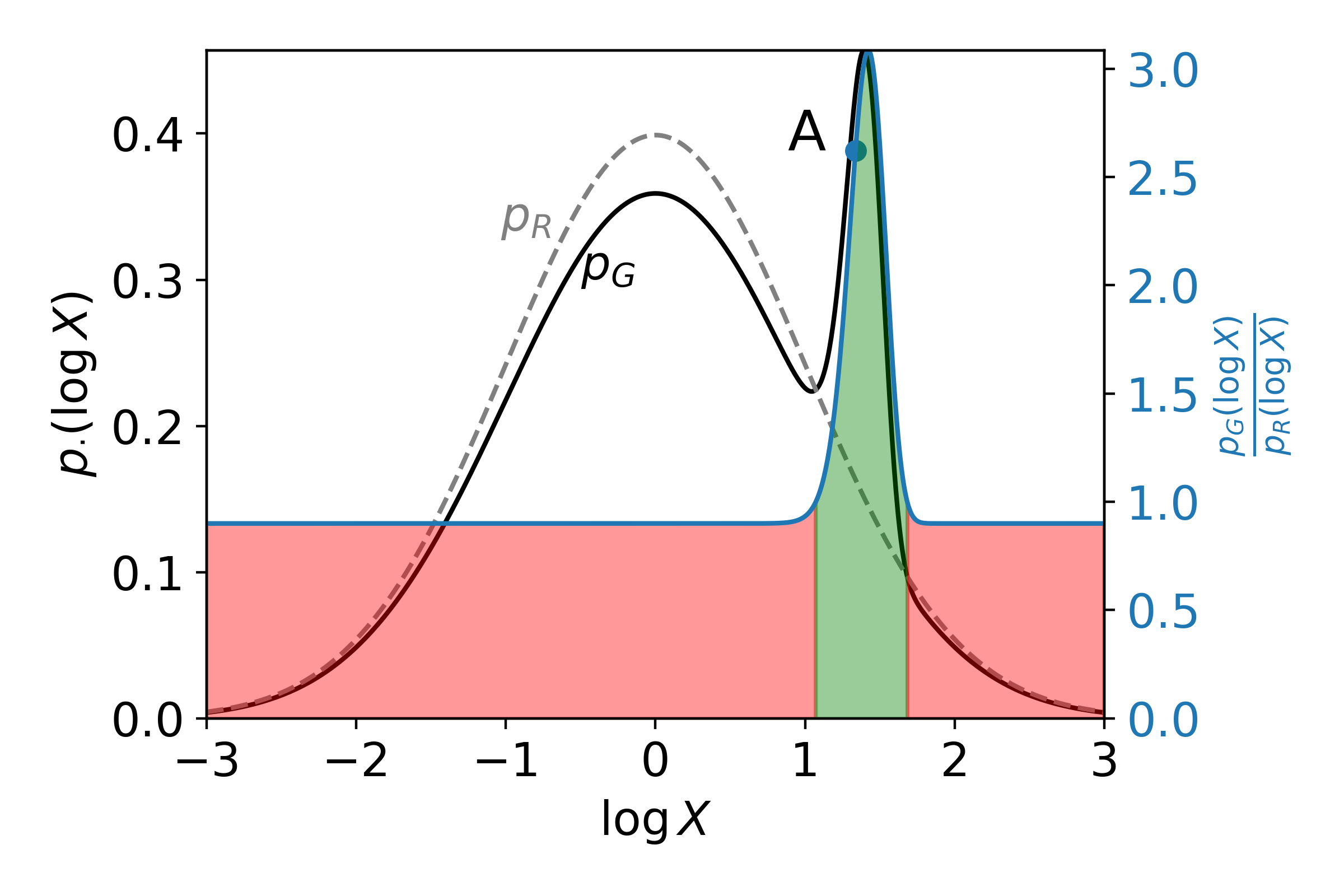}
    
    \caption{Log-transformed space}
    \end{subfigure}
    \caption{\emph{DOMIAS scores are not dependent on the feature representation.} This is the same toy example as in Figure \ref{fig:toy_example_eq1}, where we now assume the bump at $x=4$ has been caused by overfitting in the generator, s.t. this part of the space has become overrepresented w.r.t. the original distribution. DOMIAS infers MI by weighting the generative and real distribution, inferring $m=1$ ($m=0$) for green (red) areas. Note the difference with Figure \ref{fig:toy_example_eq1}: whereas MI predictions of previous works that use Eq. \ref{eq:assumption_domias} are dependent on the representation, DOMIAS scores are the same in both domains (Theorem \ref{theorem:representation}).}
    \label{fig:toy_example_eq2}
\end{figure*}

\subsection{DOMIAS: adding knowledge of the real data.} 
We need to target overfitting directly. We propose the DOMIAS framework: Detecting Overfitting for Membership Inference Attacks against Synthetic Data. 

Let us assume we know the true data distribution $p_R(X)$. We change Eq. \ref{eq:assumption_prev} to:
\begin{equation}
\label{eq:assumption_domias}
    A_{\mathrm{DOMIAS}}(x^*) = f(\frac{p_G(x^*)}{p_R(x^*)}),
\end{equation}
that is, we weight Eq. \ref{eq:assumption_prev} by the real data distribution $p_R(X)$.\footnote{This work focuses on relative scores, hence we ignore choosing $f$---see Sec. \ref{sec:discussion}.}  Figure \ref{fig:toy_example_eq2} shows the difference between DOMIAS and previous work using Eq. \ref{eq:assumption_prev}, by considering the same toy example as in Figure \ref{fig:toy_example_eq1}. Effectively, Eq. \ref{eq:assumption_domias} distinguishes between the real and generative distribution, similar in vain to global two-sample tests (e.g. see \cite{Gretton2012ATest,Arora2019ALearning, Gulrajani2019TowardsGeneralization}). The probability ratio has the advantage that (cf. e.g. probability difference) it is independent of the specific representation of the data:
\begin{theorem} \label{theorem:representation}
Let $X_G$ and $X_R$ be two random variables defined on $\mathcal{X}$, with distributions $p_G(X)$ and $p_R(X)$, s.t. $p_G\ll p_R$, i.e. $p_R$ dominates $p_G$. Let $g:\mathcal{X}\rightarrow \tilde{\mathcal{X}}, x\mapsto g(x)$ be some invertible function, and define representations $\tilde{X}_G = g(X_G)$ and $\tilde{X}_R=g(X_R)$ with respective distribution $\tilde{p}_G(\tilde{X})$ and $\tilde{p}_R(\tilde{X})$. Then $\frac{p_G(X)}{p_R(X)} = \frac{\tilde{p}_G(g(X))}{\tilde{p}_R(g(X))}$, i.e. the same score is obtained for either data representations.
\end{theorem}
\begin{proof}
Without loss of generalisation let us assume continuous variables and almost everywhere continuous $g$. Using the chain rule, we have $\tilde{p}_{\cdot}(g(x)) = \frac{p_\cdot(x)}{|J(x)|}$ with Jacobian $J(x) = \frac{dg}{dx}(x)$. Hence we see:
\begin{equation*}
    \frac{\tilde{p}_G(g(x))}{\tilde{p}_R(g(x))}= \frac{p_G(x)/|J(x)|}{p_R(x)/|J(x)|} = \frac{p_G(x)}{p_R(x)}, a.e.
\end{equation*}
as desired.
\end{proof}


\textbf{DOMIAS does not purport false privacy safety for underrepresented groups} Figure \ref{fig:toy_example_eq1}a pinpoints a problem with previous works: methods that rely on assumption Eq. \ref{eq:assumption_prev} cannot attack low-density regions. As a result, one might conclude that samples in these regions are safer. Exactly the opposite is true: in Figure \ref{fig:toy_example_eq2} we see DOMIAS infers MI successfully for these samples, whatever the representation. This is distressing, as low-density regions may correspond to underrepresented groups in the population, e.g. ethnic minorities. We will explore this further in the experimental section.

\subsection{Illustrative attacker examples} 
Any density estimator can be used for approximating $p_G(X)$ and $p_R(X)$---e.g. fitting of some parametric family, training a generative model with Monte Carlo Integration, or a deep density estimator. The choice of density estimator should largely depend whether prior knowledge is available---e.g. $p_R$ falls in some parametric family---and on the size of the datasets---for a large dataset a more powerful and more flexible density estimator can be used, whereas for little data this is not suitable as it might lead to overfitting. In the experimental section, we illustrate DOMIAS using the flow-based BNAF \citep{deCao2019BlockFlow} density estimator, chosen for its training efficiency. For the ablation study in Sec. \ref{sec:ablation} we also include a Gaussian KDE-based method as a non-parametric alternative.

\section{RELATED WORK} \label{sec:related} 
\textbf{MIAs against generative models} Most of the literature on privacy attacks is focused on discriminative models, not generative models. The few works that are concerned with generative models all focus on membership inference (MIA) \citep{Shokri2017MembershipModels}. Here we focus on works that can be applied to our attacker setting, see Table \ref{tab:attacks}.

\citet{Hayes2019LOGAN:Models} propose LOGAN, a range of MIA attacks for both white-box and black-box access to the generative model, including possible auxiliary information. Two attacks can be applied to our setting. They propose a full black-box attack without auxiliary knowledge (i.e. no reference dataset). This model trains a GAN model on the synthetic data, after which the GAN's discriminator is used to compute the score for test examples. They also propose an attack that assumes an independent test set, similar to DOMIAS' $\cD_{ref}$---see Section 4.1 \citep{Hayes2019LOGAN:Models}, discriminative setting 1 (D1). Their attacker is a simple classifier that is trained to distinguish between synthetic and test samples. \citet{Hilprecht2019MonteModels} introduce a number of attacks that focus on approximating the generator distribution at each test point. 
Implicitly, they make assumption \ref{eq:assumption_prev}, and approximate the probability by using Monte Carlo integration, i.e. counting the proportion of generated points that fall in a given neighbourhood. They do not consider the possible attacker access to a reference dataset. Choosing a suitable distance metric for determining neighbourhoods is non-trivial, however this is somewhat alleviated by choosing a better space in which to compute metrics, e.g. \citeauthor{Hilprecht2019MonteModels} show that using the Euclidean distance is much more effective when used in conjunction with Principal Component Analysis (PCA). We refer to their method as MC, for Monte Carlo integration.

\citet{Chen2019GAN-Leaks:Models} give a taxonomy of MIAs  against GANs and propose new MIA method  GAN-leaks that relies on Eq. \ref{eq:assumption_prev}. For each test point $x^*$ and some $k\in\mathbb{N}$, they sample $S^k_G = \{x_i\}_{i=1}^k$ from generator $G$ and use score $A(x^*;G) = \min_{x_i\in S^k_G} L_2(x^*, x_i)$ as an unnormalised surrogate for $p_G(x^*)$. 
They also introduce a calibrated method that uses a reference dataset $\cD_{ref}$ to train a generative reference model $G_{ref}$, giving calibrated score $A(x^*;G,k)-A(x^*;G_{ref},k)$. 
This can be interpreted as a special case of DOMIAS---Eq. \ref{eq:assumption_domias}---that approximates $p_R$ and $p_G$ with Gaussian KDEs with infinitesimal kernel width, trained on a random subset of $k$ samples from $\cD_{ref}$ and $\cD_{syn}$. At last, we emphasise that though \citep{Hayes2019LOGAN:Models, Chen2019GAN-Leaks:Models} consider $\mathcal{D}_{ref}$ too, they (i) assume this implicitly and just for one of their many models, (ii) do not properly motivate or explain the need for having $\cD_{ref}$, nor explore the effect of $n_{ref}$, and (iii) their MIAs are technically weak and perform poorly as a result, leading to incorrect conclusions on the danger of this scenario (e.g. \citet{Hayes2019LOGAN:Models} note in their experiments that their D1 model performs no better than random guessing).

\begin{table*}[bt]
    \centering
    \caption{Membership Inference attacks on generative models. (1) Underlying ML method (GAN: generative adversarial network, NN: (weighted) Nearest neighbour, KDE: kernel density estimation, MLP: multi-layer perceptron, DE: density estimator); (2) uses $\cD_{ref}$; (3) approximates Eq. \ref{eq:assumption_prev} or \ref{eq:assumption_domias}; (4) by default does not need generation access to generative model---only synthetic data itself. \textit{\textsuperscript{\textdagger}GAN-leaks calibrated is a heuristic correction to GAN-leaks, but implicitly a special case of Eq. \ref{eq:assumption_domias}.}} 
    \label{tab:attacks}
    \begin{tabular}{lcccc} \toprule
        Name & (1) & (2) & (3) & (4)\\ \midrule
        LOGAN 0\citep{Hayes2019LOGAN:Models} & GAN &  \xmark & Eq. \ref{eq:assumption_prev}& \cmark\\
        LOGAN D1 \citep{Hayes2019LOGAN:Models} & MLP & \cmark & N/A (heuristic) & \cmark\\
        MC \citep{Hilprecht2019MonteModels}& NN/KDE & \xmark & Eq. \ref{eq:assumption_prev}& \xmark\\
        GAN-leaks 0 \citep{Chen2019GAN-Leaks:Models} & NN/KDE & \xmark & Eq. \ref{eq:assumption_prev}& \xmark\\
        GAN-leaks CAL \citep{Chen2019GAN-Leaks:Models} & NN/KDE  & \cmark & \ \ Eq.  \ref{eq:assumption_domias}\textsuperscript{\textdagger} &\xmark\\ \hline
        DOMIAS (Us) & any DE & \cmark & Eq. \ref{eq:assumption_domias}& \cmark\\ \bottomrule
    \end{tabular}
\end{table*}


\textbf{Stronger attacker access assumptions} Other methods in \citep{Hayes2019LOGAN:Models, Hilprecht2019MonteModels, Chen2019GAN-Leaks:Models} make much stronger assumptions on attacker access. \citep{Hayes2019LOGAN:Models} propose multiple attacks with a subset of the training set known, which implies that there has already been a privacy breach---this is beyond the scope of this work. They also propose an attack against GANs that uses the GANs discriminator to directly compute the MIA score, but discriminators are usually not published. \citet{Chen2019GAN-Leaks:Models} propose attacks with white-box access to the generator or its latent code, but this scenario too can be easily avoided by not publishing the generative model itself. All methods in \citep{Hilprecht2019MonteModels, Chen2019GAN-Leaks:Models} assume unlimited generation access to the generator (i.e. infinitely-sized $\cD_{syn}$), which is unrealistic for a real attacker---either on-demand generation is unavailable or there is a cost associated to it that effectively limits the generation size \citep{DeCristofaro2021ALearning}. These methods can still be applied to our setting by sampling from the synthetic data directly.

\textbf{Tangential work}
The following MIA work is not compared against. \citet{Liu2019PerformingModels,Hilprecht2019MonteModels} introduce \textit{co-membership} \citep{Liu2019PerformingModels} or \textit{set MIA} \citep{Hilprecht2019MonteModels} attacks, in which the aim is to determine for a whole set of examples whether either all or none is used for training. Generally, this is an easier attack and subsumes the task of single attacks (by letting the set size be 1). 
\citet{Webster2021ThisFaces} define the \textit{identity} membership inference attack against face generation models, which aims to infer whether some person was used in the generative model (but not necessarily a specific picture of that person). This requires additional knowledge for identifying people in the first place, and does not apply to our tabular data setting. \citet{Hu2021MembershipRegions} focus on performing high-precision attacks, i.e. determining MIA for a small number of samples with high confidence. Similar to us they look at overrepresented regions in the generator output space, but their work assumes full model access (generator and discriminator) and requires a preset partitioning of the input space into regions. \citep{Zhang2022MembershipData} is similar to \citep{Hilprecht2019MonteModels}, but uses contrastive learning to embed data prior to computing distances. In higher dimensions, this can be an improvement over plain data or simpler embeddings like PCA---something already considered by \citep{Hilprecht2019MonteModels}. However, the application of contrastive learning is limited when there is no \textit{a priori} knowledge for performing augmentations, e.g. in the unstructured tabular domain. 
 
On a final note, we like to highlight the relation between MIA and the evaluation of overfitting, memorisation and generalisation of generative models. The latter is a non-trivial task, e.g. see \citep{Gretton2012ATest,Lopez-Paz2016RevisitingTests,Arora2017GeneralizationGANs,Webster2019DetectingRecovery, Gulrajani2019TowardsGeneralization}. DOMIAS targets overfitting directly and locally through Eq. \ref{eq:assumption_domias}, a high score indicating local overfitting. 
DOMIAS differs from this line of work by focusing on MIA, requiring sample-based scores. DOMIAS scores can be used for interpreting overfitting of generative models, especially in the non-image domain where visual evaluation does not work.  
 
\section{EXPERIMENTS} \label{sec:experiments} 
We perform experiments showing DOMIAS' value and use cases. In Sec. \ref{sec:domias_vs_baselines} we show how DOMIAS outperforms prior work, in Sec. \ref{sec:ablation} we explore why. Sec. \ref{sec:underrepresented} demonstrates how underrepresented groups in the population are most vulnerable to DOMIAS attack, whilst Sec. \ref{sec:generative_model_comparison} explores the vulnerability of different generative models---showcasing how DOMIAS can be used as a metric to inform synthetic data generation.  For fair evaluation, the same experimental settings are used across MIA models (including $n_{ref}$). Details on experimental settings can be found in Appendix \ref{appx:experimental_details}.\footnote{Code is available at \\ \href{https://github.com/vanderschaarlab/DOMIAS}{https://github.com/vanderschaarlab/DOMIAS}} 

\subsection{DOMIAS outperforms prior MIA methods} \label{sec:domias_vs_baselines} 
\begin{figure*}[hbt]

    \centering
    \begin{subfigure}{0.48\textwidth}
    \centering
    \includegraphics[width=0.9\textwidth]{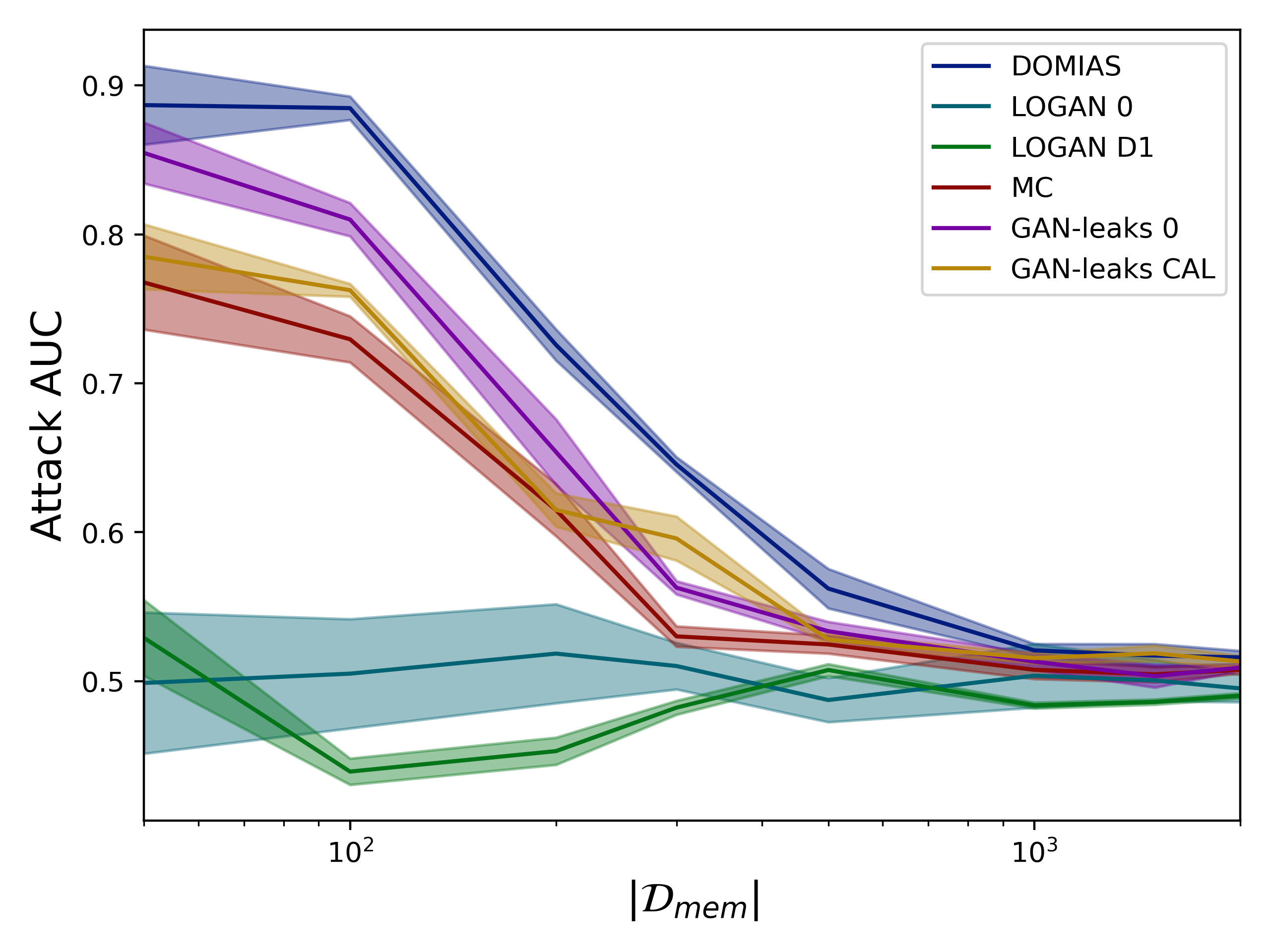}
    \end{subfigure}
    \hfill
    \begin{subfigure}{0.48\textwidth}
    \centering
    \includegraphics[width=0.9\textwidth]{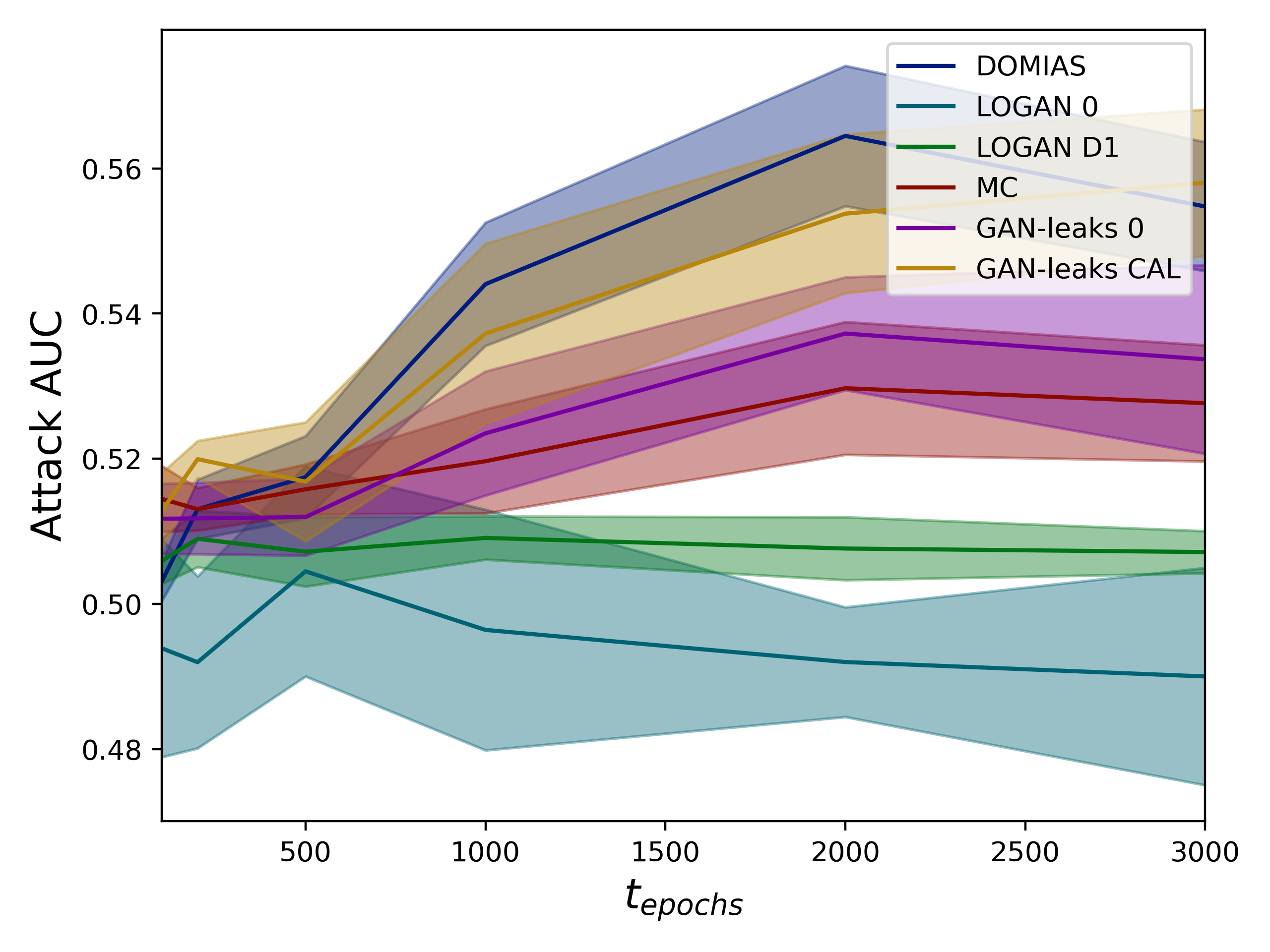}
    \end{subfigure}
    \caption{\emph{DOMIAS outperforms baselines.} MIA performance of DOMIAS and baselines versus the generative model training set size $|\cD_{mem}|$ and training time $t_{epochs}$ on the California Housing dataset. We observe how MIA AUC goes up for fewer training samples and long generative model training time, as both promote overfitting.}
    \label{fig:domias_vs_baselines}
\end{figure*}

\textbf{Set-up} We use the California Housing Dataset \citep{Pace1997SparseAutoregressions} and use TVAE \citep{Xu2019ModelingGAN} to generate synthetic data. In this experiment we vary the number of TVAE training samples $|\cD_{mem}|$ and TVAE number of training epochs. We compare DOMIAS against LOGAN 0 and LOGAN D1 \citep{Hayes2019LOGAN:Models}, MC \citep{Hilprecht2019MonteModels}, and GAN-Leaks 0 and GAN-Leaks CAL \citep{Chen2019GAN-Leaks:Models}---see Table \ref{tab:attacks}.

\textbf{DOMIAS consistently outperforms baselines}
Figure \ref{fig:domias_vs_baselines}(a) shows the MIA accuracy of DOMIAS and baselines against TVAE's synthetic dataset, as a function of the number of training samples TVAE $n_{mem}$. For small $n_{mem}$ TVAE is more likely to overfit to the data, which is reflected in the overall higher MIA accuracy. Figure \ref{fig:domias_vs_baselines}(b) shows the MIA accuracy as a function of TVAE training epochs. Again, we see TVAE starts overfitting, leading to higher MIA for large number of epochs. 

In both plots, we see DOMIAS consistently outperforms baseline methods. Similar results are seen on other datasets and generative models, see Appendix \ref{appx:additional_results}. Trivially, DOMIAS should be expected to do better than GAN-Leaks 0 and LOGAN 0, since these baseline methods do not have access to the reference dataset and are founded on the flawed assumption of Eq. \ref{eq:assumption_prev}---which exposes the privacy risk of attacker access to a reference dataset.

\begin{figure*}[hbt]
    \centering
    \begin{subfigure}{0.48\textwidth}
    \centering
    \includegraphics[width=0.9\textwidth]{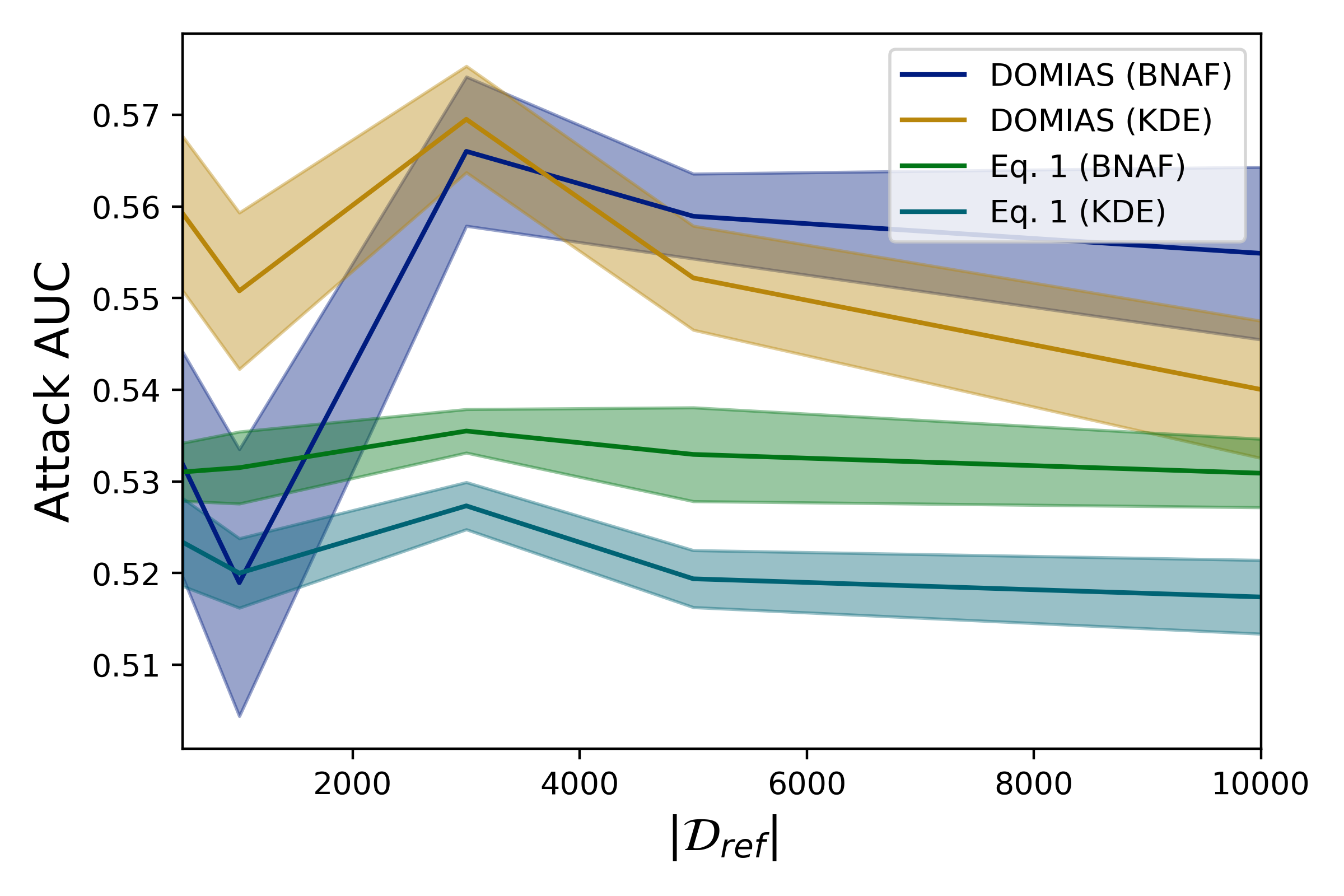}
    \end{subfigure}
    \hfill
    \begin{subfigure}{0.48\textwidth}
    \centering
    \includegraphics[width=0.9\textwidth]{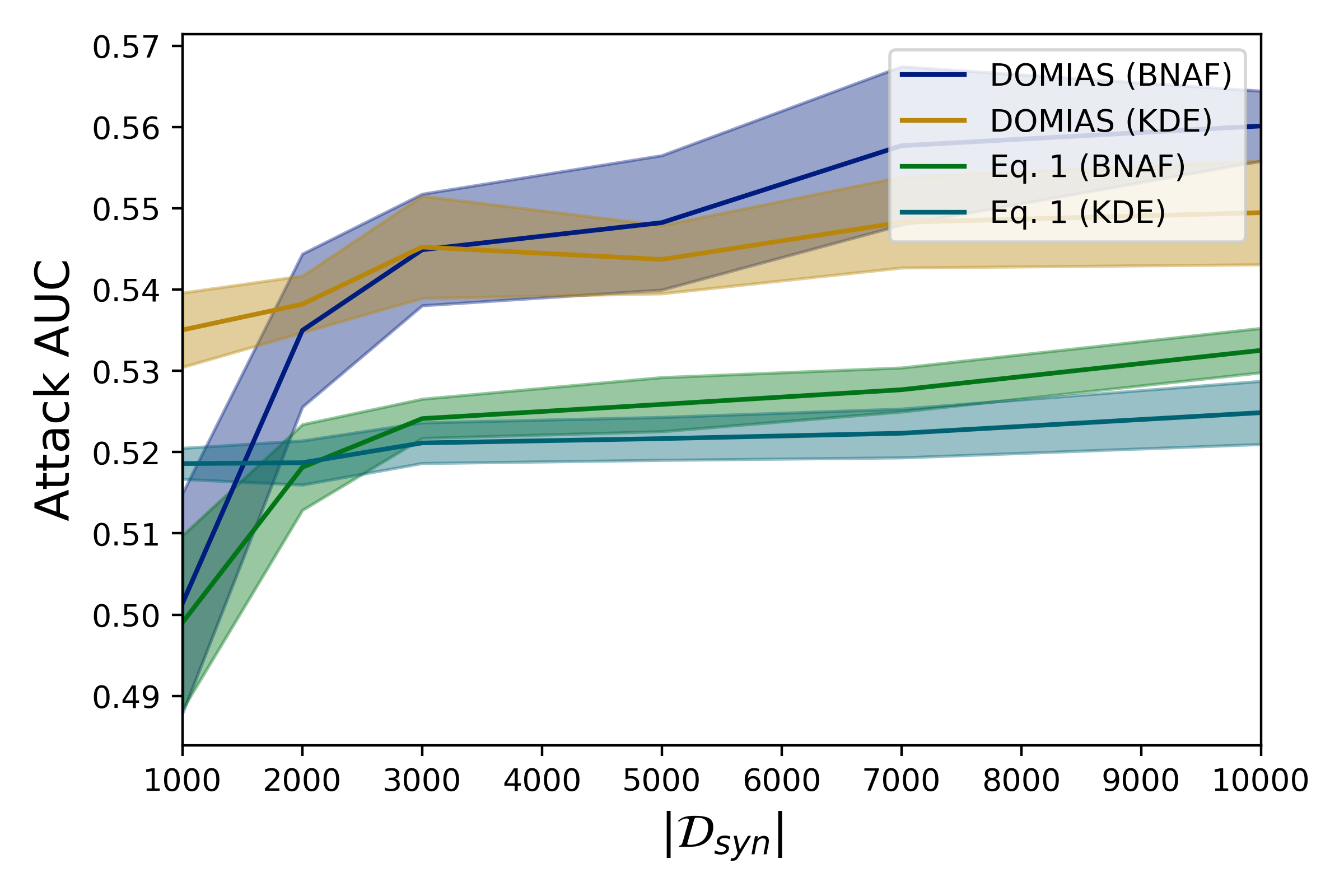}
    \end{subfigure}
    \caption{\emph{DOMIAS source of gain.} Ablation study of DOMIAS on the California Housing dataset, with attack performance as a function of the reference dataset size (left) and the synthetic dataset size (right). We see that the MIA performance of DOMIAS is largely due to assumption Eq. \ref{eq:assumption_domias} vs. Eq. \ref{eq:assumption_prev}, i.e. the value of the reference dataset. The deep flow-based density estimator delivers gains over the simpler KDE approach when enough samples are available.}
    \label{fig:results_ablation}
\end{figure*}

\subsection{Source of gain} \label{sec:ablation} 
Using the same set-up as before, we perform an ablation study on the value of i) DOMIAS' use of the reference set, and ii) the deep density estimator. For the first, we compare using the DOMIAS assumption (Eq. \ref{eq:assumption_domias}) vs the assumption employed in many previous works (Eq. \ref{eq:assumption_prev}). For the latter, we compare the results for density estimation based on the flow-based BNAF \citep{deCao2019BlockFlow} versus a Gaussian kernel density estimator---kernel width given by the heuristic from \citep{Scott1992MultivariateEstimation}. 

Figure \ref{fig:results_ablation} shows the MIA performance as a function of $n_{syn}$ and $n_{ref}$. Evidently, the source of the largest gain is the use of Eq. \ref{eq:assumption_domias} over Eq. \ref{eq:assumption_prev}. As expected, the deep density estimator gives further gains when enough data is available. For lower amounts of data, the KDE approach is more suitable. This is especially true for the approximation of $p_R$ (the denominator of Eq. \ref{eq:assumption_domias})---small noise in the approximated $p_R$ can lead to large noise in MIA scores. Also note in the right plot that MIA performance goes up with $|\cD_{syn}|$ across methods due to the better $p_G$ approximation; this motivates careful consideration for the amount of synthetic data published.

\subsection{Underrepresented group MIA vulnerability} \label{sec:underrepresented} 
\textbf{Set-up} We use a private medical dataset on heart failure, containing around $40,000$ samples with $35$ mixed-type features (see Appendix \ref{appx:experimental_details}). We generate synthetic data using TVAE \citep{Xu2019ModelingGAN}.

\begin{figure*}[bt]

    \centering
    \includegraphics[width=\textwidth]{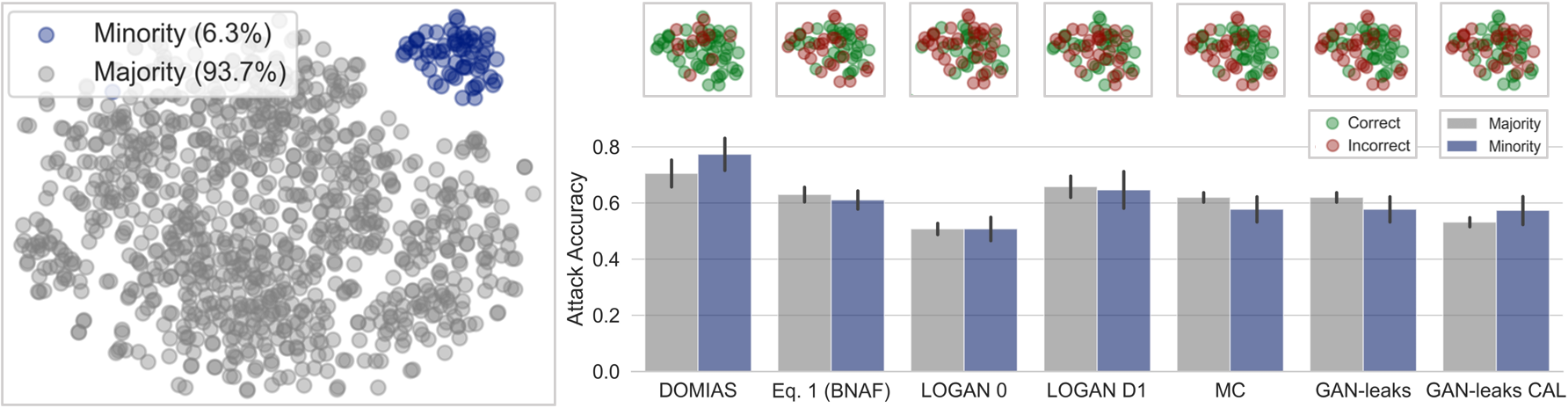}
    \caption{\emph{DOMIAS is more successful at attacking patients taking high-blood pressure medication. }(left) T-SNE plot of Heart Failure test dataset. There is a cluster of points visible in the top right corner, which upon closer inspection corresponds to subjects who take ARB medication. (right, bottom) Attacking accuracy of DOMIAS and baselines on majority and minority group (averaged over 8 runs). DOMIAS is significantly better at attacking the minority group than the general population. Except for GAN-leaks CAL, baselines fail to capture the excess privacy risk to the patients with blood pressure medication. Comparing DOMIAS with Eq. 1 (BNAF) (see Sec. \ref{sec:ablation}), we see that the minority vulnerability is largely due to the availability of the reference data. (right, top) Single run attacking success of different MIA methods on these underrepresented samples; correctly inferred membership in green, incorrectly inferred in red.}
    \label{fig:vulnerable}
\end{figure*}


\textbf{Minority groups are most vulnerable to DOMIAS attack} As seen in Sec. \ref{sec:method}, the assumption underlying previous work (Eq. \ref{eq:assumption_prev}) will cause these methods to never infer membership for low-density regions. This is problematic, as it gives a false sense of security for these groups---which are likely to correspond to underrepresented groups.

The left side of Figure \ref{fig:vulnerable} displays a T-SNE embedding of the Heart Failure dataset, showing one clear minority group, drawn in blue, which corresponds to patients that are on high-blood pressure medication---specifically, Angiotensin II receptor blockers. The right side of Figure \ref{fig:vulnerable} shows the performance of different MIA models. DOMIAS is significantly better at attacking this vulnerable group compared to the overall population, as well as compared to other baselines. This is not entirely surprising; generative models are prone to overfitting regions with few samples. Moreover, this aligns well with supervised learning literature that finds additional vulnerability of low-density regions, e.g. \citep{Kulynych2019DisparateAttacks, Bagdasaryan2019DifferentialAccuracy}.  Importantly, most MIA baselines give the false pretense that this minority group is \textit{less vulnerable}. Due to the correspondence of low-density regions and underrepresented groups, \emph{these results strongly urge further research into privacy protection of low-density regions when generating synthetic data.} 

\begin{figure}[hbt]
    \centering
    \includegraphics[width=0.5\textwidth]{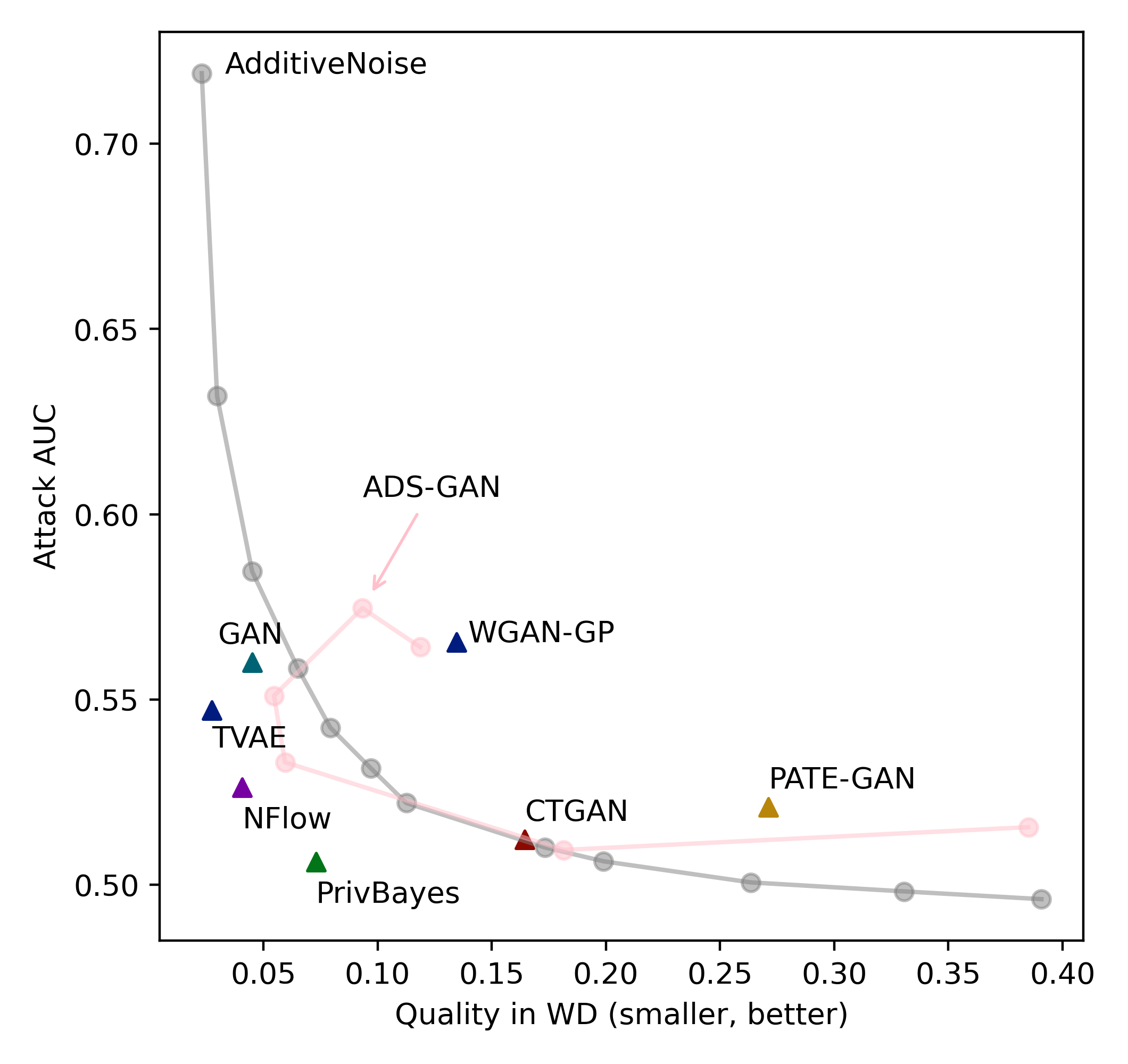}
    \caption{\emph{DOMIAS can be used to quantify synthetic data MIA vulnerability.} We plot the synthetic data quality versus DOMIAS AUC for different generative models on the California Housing dataset. There is a clear trade-off: depending on the tolerated MIA vulnerability, different synthetic datasets are best.}
    \label{fig:generative_model_comparison}
\end{figure}

\subsection{DOMIAS informs generative modelling decisions} \label{sec:generative_model_comparison} 
\textbf{Set-up} Again we use the California Housing dataset, this time generating synthetic data using different generative models. We evaluate the quality and MIA vulnerability of GAN, \citep{Goodfellow2014GenerativeNetworks}, WGAN-GP \citep{Arjovsky2017WassersteinNetworks,Gulrajani2017ImprovedGANs}, CTGAN and TVAE \citep{Xu2019ModelingGAN}, 
NFlow \citep{Durkan2019NeuralFlows}, PATE-GAN \citep{Jordon2019PATE-GAN:Guarantees}, PrivBayes \citep{Zhang2017Privbayes:Networks}, and ADS-GAN \citep{Yoon2020Anonymizationads-gan}. As a baseline, we also include the anonymization method of sampling from training data and adding Gaussian noise. For ADS-GAN and the additive noise model, we vary the privacy level by raising the hyperparameter $\lambda$ and noise variance, respectively. Results for other attackers are found in Appendix \ref{appx:additional_results}.

\textbf{DOMIAS quantifies MIA vulnerability}
Figure \ref{fig:generative_model_comparison} presents the DOMIAS MIA AUC against the data quality (in terms of Wasserstein Distance to an independent hold-out set), averaged over eight runs. We see a clear privacy-utility trade-off, with the additive noise model giving a clean baseline. The NeurIPS 2020 Synthetic Data competition \citep{Jordon2021Hide-and-SeekRe-identification} concluded that disappointingly, adding noise usually outperformed generative models in terms of the privacy-utility trade-off. Though we find this is true for WGAN-GP, PATE-GAN and CTGAN---which fall on the right side of the additive noise curve---other methods do yield better synthetic datasets. 

ADS-GAN is based on WGAN-GP, hence for small $\lambda$ (the privacy regularizer) it gets a similar score. Increasing $\lambda$ promotes a higher distance between generated and training data, hence this reduces vulnerability. At first, it also leads to an increase in quality---raising $\lambda$ leads to lower overfitting---but when $\lambda$ increases further the generative distribution is distorted to the point that quality is significantly reduced. In contrast to \citep{Hilprecht2019MonteModels}, we do not find evidence that VAEs are more vulnerable to MIAs than GANs. The Pareto frontier is given by the additive noise method, TVAE, NFlow and PrivBayes, hence the best synthetic data model will be one of these, depending on the privacy requirements.

\section{DISCUSSION} \label{sec:discussion} 
\textbf{DOMIAS use cases} DOMIAS is primarily a tool for evaluating and interpreting generative model privacy. The overall DOMIAS attacking success is a metric for MIA vulnerability, and may hence guide generative model design choices---e.g. choosing privacy parameters---or aid evaluation---including for competitions like \citep{Jordon2021Hide-and-SeekRe-identification}. Since DOMIAS provides a sample-wise metric, its scores can also provide insight into privacy and overfitting of specific samples or regions in space---as seen in Sec. \ref{sec:underrepresented}. Future work may adopt DOMIAS for active privacy protection, e.g. as a loss during training or as an auditing method post-training---removing samples that are likely overfitted. 

\textbf{Underrepresented groups are more vulnerable to MIA attacks} Generative models are more likely to overfit low-density regions, and we have seen DOMIAS is indeed more successful at attacking these samples. This is distressing, since these regions can correspond to underrepresented groups in the population. Similar results have been found in supervised learning literature, e.g. \citep{Kulynych2019DisparateAttacks, Bagdasaryan2019DifferentialAccuracy}. Protecting against this vulnerability is a trade-off, as outliers in data can often be of interest to downstream research. It is advisable data publishers quantify the excess MIA risk to specific subgroups.

\textbf{Attacker calibration} In practice, it will often be unknown how much of the test data was used for training. Just like related works, we have ignored this. This challenge is equivalent to choosing a suitable threshold, or suitable $f$ in Eq. \ref{eq:assumption_domias} and relates closely to calibration of the attacker model, which is challenging for MIA since---to an attacker---usually no ground-truth labels are available. Future work can explore assumptions or settings that could enable calibrated attacks. In Appendix \ref{appx:high_precision attacks} we include results for high-precision attacks.

\textbf{High-dimensionality and image data} Traditional density estimation methods (e.g. KDE) perform notoriously poorly in high dimensions. Recent years have seen a rise in density estimation methods that challenge this conception. Domain-specific density estimators, e.g. that define density on lower-dimensional embeddings, can be readily used in DOMIAS. We include preliminary results for the high-dimensional CelebA image dataset in Appendix \ref{sec:celeba}.

\textbf{Training data size} We have seen that for large number of training samples, the performance of all attackers goes down to almost 0.5. The same is observed for large generative image models, Appendix \ref{sec:celeba}. This is reassuring for synthetic data publishers, for whom this indicates a relatively low privacy risk globally. However, global metrics may hide potential high-precision attacks on a small number of individuals, see Appendix \ref{appx:high_precision attacks}.

\textbf{Availability of reference dataset} DOMIAS assumes the presence of a reference dataset that enables approximating the true distribution $p_R(X)$. In case there is not sufficient data for the latter, more prior knowledge can be included in the parametrisation of $p_R$; e.g. choose $p_R(X)$ to lie in a more restrictive parametric family. Even in the absence of any data $\cD_{ref}$, an informed prior (e.g. Gaussian) based on high-level statistics can already improve upon related works that rely on assumption Eq. \ref{eq:assumption_prev}---see Appendix \ref{appx:gaussian_prior} for results. In Appendix \ref{appx:distributional_shift} we include further experiments with distributional shifts between the $\cD_{ref}$ and $\cD_{mem}$, in which we find that even with moderate shifts the use of a reference dataset is beneficial.


\textbf{Publishing guidelines} Synthetic data does not guarantee privacy, however the risk of MIA attacks can be lessened when synthetic data is published considerately. Publishing just the synthetic data---and not the generative model---will in most cases be sufficient for downstream research, while avoiding more specialised attacks that use additional knowledge. Further consideration is required with the amount of data published: increasing the amount of synthetic data leads to higher privacy vulnerability (Figure \ref{fig:results_ablation}b and see \citep{Gretton2012ATest}). Though the amount of required synthetic data is entirely dependent on the application, DOMIAS can aid in finding the right privacy-utility trade-off.

\textbf{Societal impact} We believe DOMIAS can provide significant benefits to the future privacy of synthetic data, and that these benefits outweigh the risk DOMIAS poses as a more successful MIA method. On a different note, we highlight that success of DOMIAS implies privacy is not preserved, but not vice versa. Specifically, DOMIAS should not be used as a certificate for data privacy. Finally, we hope the availability of a reference dataset is a setting that will be considered in more ML privacy work, as we believe this is more realistic in practice than many more popular MIA assumptions (e.g. white-box generator), yet still poses significant privacy risks.

\subsubsection*{Acknowledgements}
We would like to thank the Office of Navel Research UK, who funded this research.


\begin{thebibliography}{46}
\providecommand{\natexlab}[1]{#1}
\providecommand{\url}[1]{\texttt{#1}}
\expandafter\ifx\csname urlstyle\endcsname\relax
  \providecommand{\doi}[1]{doi: #1}\else
  \providecommand{\doi}{doi: \begingroup \urlstyle{rm}\Url}\fi

\bibitem[Carlini et~al.(2018)Carlini, Liu, Erlingsson, Kos, and
  Song]{Carlini2018TheNetworks}
Nicholas Carlini, Chang Liu, Úlfar Erlingsson, Jernej Kos, and Dawn Song.
\newblock {The Secret Sharer: Evaluating and Testing Unintended Memorization in
  Neural Networks}.
\newblock \emph{Proceedings of the 28th USENIX Security Symposium}, pages
  267--284, 2 2018.
\newblock URL \url{https://arxiv.org/abs/1802.08232v3}.

\bibitem[Jordon et~al.(2021)Jordon, Jarrett, Saveliev, Yoon, Elbers, Thoral,
  Ercole, Zhang, Belgrave, and van~der
  Schaar]{Jordon2021Hide-and-SeekRe-identification}
James Jordon, Daniel Jarrett, Evgeny Saveliev, Jinsung Yoon, Paul Elbers,
  Patrick Thoral, Ari Ercole, Cheng Zhang, Danielle Belgrave, and Mihaela
  van~der Schaar.
\newblock {Hide-and-Seek Privacy Challenge: Synthetic Data Generation vs.
  Patient Re-identification}.
\newblock In Hugo~Jair Escalante and Katja Hofmann, editors, \emph{Proceedings
  of the NeurIPS 2020 Competition and Demonstration Track}, volume 133 of
  \emph{Proceedings of Machine Learning Research}, pages 206--215. PMLR, 2
  2021.
\newblock URL \url{https://proceedings.mlr.press/v133/jordon21a.html}.

\bibitem[Alaa et~al.(2022)Alaa, van Breugel, Saveliev, and van~der
  Schaar]{Alaa2022HowModels}
Ahmed~M. Alaa, Boris van Breugel, Evgeny Saveliev, and Mihaela van~der Schaar.
\newblock {How Faithful is your Synthetic Data? Sample-level Metrics for
  Evaluating and Auditing Generative Models}.
\newblock In \emph{Internal Conference on Machine Learning}, pages 290--306, 2
  2022.
\newblock URL \url{https://arxiv.org/abs/2102.08921v1}.

\bibitem[Dwork and Roth(2014)]{Dwork2014ThePrivacy}
Cynthia Dwork and Aaron Roth.
\newblock {The Algorithmic Foundations of Differential Privacy}.
\newblock \emph{Foundations and Trends in Theoretical Computer Science},
  9\penalty0 (3-4):\penalty0 211--487, 8 2014.
\newblock ISSN 15513068.
\newblock \doi{10.1561/0400000042}.
\newblock URL \url{https://dl.acm.org/doi/abs/10.1561/0400000042}.

\bibitem[Ho et~al.(2021)Ho, Qu, Gu, Gao, Li, and Xiang]{Ho2021DP-GAN:Nets}
Stella Ho, Youyang Qu, Bruce Gu, Longxiang Gao, Jianxin Li, and Yong Xiang.
\newblock {DP-GAN: Differentially private consecutive data publishing using
  generative adversarial nets}.
\newblock \emph{Journal of Network and Computer Applications}, 185:\penalty0
  103066, 7 2021.
\newblock ISSN 1084-8045.
\newblock \doi{10.1016/J.JNCA.2021.103066}.

\bibitem[Torkzadehmahani et~al.(2020)Torkzadehmahani, Kairouz, and
  Paten]{Torkzadehmahani2020DP-CGAN:Generation}
Reihaneh Torkzadehmahani, Peter Kairouz, and Benedict Paten.
\newblock {DP-CGAN: Differentially Private Synthetic Data and Label
  Generation}.
\newblock \emph{IEEE Computer Society Conference on Computer Vision and Pattern
  Recognition Workshops}, 2019-June:\penalty0 98--104, 1 2020.
\newblock ISSN 21607516.
\newblock \doi{10.48550/arxiv.2001.09700}.
\newblock URL \url{https://arxiv.org/abs/2001.09700v1}.

\bibitem[Chen et~al.(2020)Chen, Orekondy, and
  Fritz]{Chen2020GS-WGAN:Generators}
Dingfan Chen, Tribhuvanesh Orekondy, and Mario Fritz.
\newblock {GS-WGAN: A Gradient-Sanitized Approach for Learning Differentially
  Private Generators}.
\newblock \emph{Advances in Neural Information Processing Systems},
  2020-December, 6 2020.
\newblock ISSN 10495258.
\newblock URL \url{https://arxiv.org/abs/2006.08265v2}.

\bibitem[Jordon et~al.(2019)Jordon, Yoon, and
  Schaar]{Jordon2019PATE-GAN:Guarantees}
James Jordon, Jinsung Yoon, and M~Schaar.
\newblock {PATE-GAN: Generating Synthetic Data with Differential Privacy
  Guarantees}.
\newblock In \emph{7th International Conference on Learning Representations,
  ICLR 2019}, 2019.

\bibitem[Long et~al.(2019)Long, Wang, Yang, Kailkhura, Zhang, Gunter, and
  Li]{Long2019G-PATE:Discriminators}
Yunhui Long, Boxin Wang, Zhuolin Yang, Bhavya Kailkhura, Aston Zhang, Carl~A.
  Gunter, and Bo~Li.
\newblock {G-PATE: Scalable Differentially Private Data Generator via Private
  Aggregation of Teacher Discriminators}.
\newblock In \emph{Advances in Neural Information Processing Systems}, 6 2019.
\newblock \doi{10.48550/arxiv.1906.09338}.
\newblock URL \url{https://arxiv.org/abs/1906.09338v2}.

\bibitem[Wang et~al.(2021)Wang, Wu, Long, Z{\"{u}}rich~Z{\"{u}}rich, Ce~Zhang,
  Bo~Li, Rimanic, Zhang, and Li]{Wang2021DataLens:Aggregation}
Boxin Wang, Fan Wu, Yunhui Long, Eth Z{\"{u}}rich~Z{\"{u}}rich, Switzerland
  Ce~Zhang, Switzerland Bo~Li, Luka Rimanic, Ce~Zhang, and Bo~Li.
\newblock {DataLens: Scalable Privacy Preserving Training via Gradient
  Compression and Aggregation}.
\newblock \emph{Proceedings of the ACM Conference on Computer and
  Communications Security}, pages 2146--2168, 3 2021.
\newblock \doi{10.1145/3460120.3484579}.
\newblock URL \url{http://arxiv.org/abs/2103.11109
  http://dx.doi.org/10.1145/3460120.3484579}.

\bibitem[Cao et~al.(2021)Cao, Bie, Vahdat, Fidler, and
  Kreis]{Cao2021DontDivergence}
Tianshi Cao, Alex Bie, Arash Vahdat, Sanja Fidler, and Karsten Kreis.
\newblock {Don't Generate Me: Training Differentially Private Generative Models
  with Sinkhorn Divergence}.
\newblock In \emph{Advances in Neural Information Processing Systems}, 11 2021.
\newblock URL \url{https://arxiv.org/abs/2111.01177v2}.

\bibitem[Rahman et~al.(2018)Rahman, Rahman, Lagan{\`{i}}, Mohammed, and
  Wang]{Rahman2018MembershipModel}
Atiqur Rahman, Tanzila Rahman, Robert Lagan{\`{i}}, Noman Mohammed, and Yang
  Wang.
\newblock {Membership Inference Attack against Differentially Private Deep
  Learning Model}.
\newblock \emph{TRANSACTIONS ON DATA PRIVACY}, 11:\penalty0 61--79, 2018.

\bibitem[Jayaraman and Evans(2019)]{Jayaraman2019EvaluatingPractice}
Bargav Jayaraman and David Evans.
\newblock {Evaluating Differentially Private Machine Learning in Practice}.
\newblock \emph{Proceedings of the 28th USENIX Security Symposium}, pages
  1895--1912, 2 2019.
\newblock \doi{10.48550/arxiv.1902.08874}.
\newblock URL \url{https://arxiv.org/abs/1902.08874v4}.

\bibitem[Yoon et~al.(2020)Yoon, Drumright, and Van
  Der~Schaar]{Yoon2020Anonymizationads-gan}
Jinsung Yoon, Lydia~N. Drumright, and Mihaela Van Der~Schaar.
\newblock {Anonymization through data synthesis using generative adversarial
  networks (ADS-GAN)}.
\newblock \emph{IEEE Journal of Biomedical and Health Informatics}, 24\penalty0
  (8):\penalty0 2378--2388, 8 2020.
\newblock ISSN 21682208.
\newblock \doi{10.1109/JBHI.2020.2980262}.

\bibitem[Rigaki and Garcia(2020)]{Rigaki2020ALearning}
Maria Rigaki and Sebastian Garcia.
\newblock {A Survey of Privacy Attacks in Machine Learning}.
\newblock \emph{arXiv preprint arXiv:2007.07646}, 7 2020.
\newblock ISSN 2331-8422.
\newblock URL \url{https://arxiv.org/abs/2007.07646v2}.

\bibitem[Shokri et~al.(2017)Shokri, Stronati, Song, and
  Shmatikov]{Shokri2017MembershipModels}
Reza Shokri, Marco Stronati, Congzheng Song, and Vitaly Shmatikov.
\newblock {Membership Inference Attacks Against Machine Learning Models}.
\newblock \emph{Proceedings - IEEE Symposium on Security and Privacy}, pages
  3--18, 6 2017.
\newblock \doi{10.1109/SP.2017.41}.

\bibitem[Hu et~al.(2022)Hu, Salcic, Sun, Dobbie, Yu, and
  Zhang]{Hu2022MembershipSurvey}
Hongsheng Hu, Zoran Salcic, Lichao Sun, Gillian Dobbie, Philip~S. Yu, and Xuyun
  Zhang.
\newblock {Membership Inference Attacks on Machine Learning: A Survey}.
\newblock \emph{ACM Computing Surveys}, 3 2022.
\newblock URL \url{http://arxiv.org/abs/2103.07853}.

\bibitem[De~Cristofaro(2021)]{DeCristofaro2021ALearning}
Emiliano De~Cristofaro.
\newblock {A Critical Overview of Privacy in Machine Learning}.
\newblock \emph{IEEE Security and Privacy}, 19\penalty0 (4):\penalty0 19--27, 7
  2021.
\newblock ISSN 15584046.
\newblock \doi{10.1109/MSEC.2021.3076443}.

\bibitem[Liu et~al.(2019)Liu, Xiao, Li, and Gao]{Liu2019PerformingModels}
Kin~Sum Liu, Chaowei Xiao, Bo~Li, and Jie Gao.
\newblock {Performing Co-Membership Attacks Against Deep Generative Models}.
\newblock \emph{Proceedings - IEEE International Conference on Data Mining,
  ICDM}, 2019-November:\penalty0 459--467, 5 2019.
\newblock ISSN 15504786.
\newblock \doi{10.1109/ICDM.2019.00056}.
\newblock URL \url{https://arxiv.org/abs/1805.09898v3}.

\bibitem[Hayes et~al.(2019)Hayes, Melis, Danezis, and
  De~Cristofaro]{Hayes2019LOGAN:Models}
Jamie Hayes, Luca Melis, George Danezis, and Emiliano De~Cristofaro.
\newblock {LOGAN: Membership Inference Attacks Against Generative Models}.
\newblock \emph{Proceedings on Privacy Enhancing Technologies}, 2019\penalty0
  (1):\penalty0 133--152, 2019.
\newblock \doi{10.2478/popets-2019-0008}.
\newblock URL \url{https://arxiv.org/abs/1705.07663}.

\bibitem[Hilprecht et~al.(2019)Hilprecht, H{\"{a}}rterich, and
  Bernau]{Hilprecht2019MonteModels}
Benjamin Hilprecht, Martin H{\"{a}}rterich, and Daniel Bernau.
\newblock {Monte Carlo and Reconstruction Membership Inference Attacks against
  Generative Models}.
\newblock In \emph{Proceedings on Privacy Enhancing Technologies}, volume 2019,
  5 2019.
\newblock \doi{10.2478/popets-2019-0067}.

\bibitem[Chen et~al.(2019)Chen, Yu, Zhang, and Fritz]{Chen2019GAN-Leaks:Models}
Dingfan Chen, Ning Yu, Yang Zhang, and Mario Fritz.
\newblock {GAN-Leaks: A Taxonomy of Membership Inference Attacks against
  Generative Models}.
\newblock \emph{Proceedings of the ACM Conference on Computer and
  Communications Security}, pages 343--362, 9 2019.
\newblock ISSN 15437221.
\newblock \doi{10.1145/3372297.3417238}.
\newblock URL \url{https://arxiv.org/abs/1909.03935v3}.

\bibitem[Ye et~al.(2021)Ye, Maddi, Murakonda, and Shokri]{Ye2021EnhancedModels}
Jiayuan Ye, Aadyaa Maddi, Sasi~Kumar Murakonda, and Reza Shokri.
\newblock {Enhanced Membership Inference Attacks against Machine Learning
  Models}.
\newblock \emph{arXiv preprint arXiv:2111.09679}, 11 2021.
\newblock \doi{10.48550/arxiv.2111.09679}.
\newblock URL \url{https://arxiv.org/abs/2111.09679v3}.

\bibitem[Gretton et~al.(2012)Gretton, Borgwardt, Rasch, Smola, Sch{\"{o}}lkopf,
  and Smola]{Gretton2012ATest}
Arthur Gretton, Karsten~M Borgwardt, Malte~J Rasch, Alexander Smola, Bernhard
  Sch{\"{o}}lkopf, and Alexander Smola.
\newblock {A Kernel Two-Sample Test}.
\newblock \emph{Journal of Machine Learning Research}, 13:\penalty0 723--773,
  2012.
\newblock URL \url{www.gatsby.ucl.ac.uk/}.

\bibitem[Arora et~al.(2019)Arora, Khandeparkar, Khodak, Plevrakis, and
  Saunshi]{Arora2019ALearning}
Sanjeev Arora, Hrishikesh Khandeparkar, Mikhail Khodak, Orestis Plevrakis, and
  Nikunj Saunshi.
\newblock {A Theoretical Analysis of Contrastive Unsupervised Representation
  Learning}.
\newblock \emph{arXiv preprint arXiv:1902.09229}, 2019.

\bibitem[Gulrajani et~al.(2019)Gulrajani, Raffel, and
  Metz]{Gulrajani2019TowardsGeneralization}
Ishaan Gulrajani, Colin Raffel, and Luke Metz.
\newblock {Towards GAN Benchmarks Which Require Generalization}.
\newblock \emph{7th International Conference on Learning Representations, ICLR
  2019}, 1 2019.
\newblock \doi{10.48550/arxiv.2001.03653}.
\newblock URL \url{https://arxiv.org/abs/2001.03653v1}.

\bibitem[de~Cao et~al.(2019)de~Cao, Aziz, and Titov]{deCao2019BlockFlow}
Nicola de~Cao, Wilker Aziz, and Ivan Titov.
\newblock {Block Neural Autoregressive Flow}.
\newblock \emph{35th Conference on Uncertainty in Artificial Intelligence, UAI
  2019}, 4 2019.
\newblock \doi{10.48550/arxiv.1904.04676}.
\newblock URL \url{https://arxiv.org/abs/1904.04676v1}.

\bibitem[Webster et~al.(2021)Webster, Rabin, Simon, and
  Jurie]{Webster2021ThisFaces}
Ryan Webster, Julien Rabin, Loic Simon, and Frederic Jurie.
\newblock {This Person (Probably) Exists. Identity Membership Attacks Against
  GAN Generated Faces}.
\newblock \emph{arXiv preprint arXiv:2107.06018}, 7 2021.
\newblock \doi{10.48550/arxiv.2107.06018}.
\newblock URL \url{https://arxiv.org/abs/2107.06018v1}.

\bibitem[Hu and Pang(2021)]{Hu2021MembershipRegions}
Hailong Hu and Jun Pang.
\newblock {Membership Inference Attacks against GANs by Leveraging
  Over-representation Regions}.
\newblock \emph{Proceedings of the ACM Conference on Computer and
  Communications Security}, pages 2387--2389, 11 2021.
\newblock ISSN 15437221.
\newblock \doi{10.1145/3460120.3485338}.
\newblock URL \url{https://doi.org/10.1145/3460120.3485338}.

\bibitem[Zhang et~al.(2022)Zhang, Yan, and Malin]{Zhang2022MembershipData}
Ziqi Zhang, Chao Yan, and Bradley~A. Malin.
\newblock {Membership inference attacks against synthetic health data}.
\newblock \emph{Journal of Biomedical Informatics}, 125:\penalty0 103977, 1
  2022.
\newblock ISSN 1532-0464.
\newblock \doi{10.1016/J.JBI.2021.103977}.

\bibitem[Lopez-Paz and Oquab(2016)]{Lopez-Paz2016RevisitingTests}
David Lopez-Paz and Maxime Oquab.
\newblock {Revisiting Classifier Two-Sample Tests}.
\newblock \emph{5th International Conference on Learning Representations, ICLR
  2017 - Conference Track Proceedings}, 10 2016.
\newblock \doi{10.48550/arxiv.1610.06545}.
\newblock URL \url{https://arxiv.org/abs/1610.06545v4}.

\bibitem[Arora et~al.(2017)Arora, Ge, Liang, Ma, and
  Zhang]{Arora2017GeneralizationGANs}
Sanjeev Arora, Rong Ge, Yingyu Liang, Tengyu Ma, and Yi~Zhang.
\newblock {Generalization and Equilibrium in Generative Adversarial Nets
  (GANs)}.
\newblock \emph{34th International Conference on Machine Learning, ICML 2017},
  1:\penalty0 322--349, 3 2017.
\newblock \doi{10.48550/arxiv.1703.00573}.
\newblock URL \url{https://arxiv.org/abs/1703.00573v5}.

\bibitem[Webster et~al.(2019)Webster, Rabin, Simon, and
  Jurie]{Webster2019DetectingRecovery}
Ryan Webster, Julien Rabin, Loic Simon, and Frederic Jurie.
\newblock {Detecting overfitting of deep generative networks via latent
  recovery}.
\newblock \emph{Proceedings of the IEEE Computer Society Conference on Computer
  Vision and Pattern Recognition}, 2019-June:\penalty0 11265--11274, 6 2019.
\newblock ISSN 10636919.
\newblock \doi{10.1109/CVPR.2019.01153}.

\bibitem[Pace and Barry(1997)]{Pace1997SparseAutoregressions}
R.~Kelley Pace and Ronald Barry.
\newblock {Sparse spatial autoregressions}.
\newblock \emph{Statistics {\&} Probability Letters}, 33\penalty0 (3):\penalty0
  291--297, 5 1997.
\newblock ISSN 0167-7152.
\newblock \doi{10.1016/S0167-7152(96)00140-X}.

\bibitem[Xu et~al.(2019{\natexlab{a}})Xu, Skoularidou, Cuesta-Infante, and
  Veeramachaneni]{Xu2019ModelingGAN}
Lei Xu, Maria Skoularidou, Alfredo Cuesta-Infante, and Kalyan Veeramachaneni.
\newblock {Modeling Tabular data using Conditional GAN}.
\newblock \emph{Advances in Neural Information Processing Systems}, 32, 7
  2019{\natexlab{a}}.
\newblock ISSN 10495258.
\newblock URL \url{https://arxiv.org/abs/1907.00503v2}.

\bibitem[Scott(1992)]{Scott1992MultivariateEstimation}
David~W. Scott.
\newblock \emph{{Multivariate Density Estimation}}.
\newblock Wiley Series in Probability and Statistics. Wiley, 8 1992.
\newblock ISBN 9780471547709.
\newblock \doi{10.1002/9780470316849}.
\newblock URL
  \url{https://onlinelibrary.wiley.com/doi/book/10.1002/9780470316849}.

\bibitem[Kulynych et~al.(2019)Kulynych, Yaghini, Cherubin, Veale, and
  Troncoso]{Kulynych2019DisparateAttacks}
Bogdan Kulynych, Mohammad Yaghini, Giovanni Cherubin, Michael Veale, and
  Carmela Troncoso.
\newblock {Disparate Vulnerability to Membership Inference Attacks}.
\newblock \emph{Proceedings on Privacy Enhancing Technologies}, 2022\penalty0
  (1):\penalty0 460--480, 6 2019.
\newblock \doi{10.48550/arxiv.1906.00389}.
\newblock URL \url{https://arxiv.org/abs/1906.00389v4}.

\bibitem[Bagdasaryan et~al.(2019)Bagdasaryan, Poursaeed, and
  Shmatikov]{Bagdasaryan2019DifferentialAccuracy}
Eugene Bagdasaryan, Omid Poursaeed, and Vitaly Shmatikov.
\newblock {Differential Privacy Has Disparate Impact on Model Accuracy}.
\newblock \emph{Advances in Neural Information Processing Systems}, 32, 5 2019.
\newblock ISSN 10495258.
\newblock \doi{10.48550/arxiv.1905.12101}.
\newblock URL \url{https://arxiv.org/abs/1905.12101v2}.

\bibitem[Goodfellow et~al.(2014)Goodfellow, Pouget-Abadie, Mirza, Xu,
  Warde-Farley, Ozair, Courville, and Bengio]{Goodfellow2014GenerativeNetworks}
Ian Goodfellow, Jean Pouget-Abadie, Mehdi Mirza, Bing Xu, David Warde-Farley,
  Sherjil Ozair, Aaron Courville, and Y~Bengio.
\newblock {Generative Adversarial Networks}.
\newblock \emph{Advances in Neural Information Processing Systems}, 3, 6 2014.
\newblock \doi{10.1145/3422622}.

\bibitem[Arjovsky et~al.(2017)Arjovsky, Chintala, and
  Bottou]{Arjovsky2017WassersteinNetworks}
Martin Arjovsky, Soumith Chintala, and Léon Bottou.
\newblock {Wasserstein Generative Adversarial Networks}.
\newblock In Doina Precup and Yee~Whye Teh, editors, \emph{Proceedings of the
  34th International Conference on Machine Learning}, volume~70 of
  \emph{Proceedings of Machine Learning Research}, pages 214--223. PMLR, 6
  2017.
\newblock URL \url{http://proceedings.mlr.press/v70/arjovsky17a.html}.

\bibitem[Gulrajani et~al.(2017)Gulrajani, Ahmed, Arjovsky, Dumoulin, and
  Courville]{Gulrajani2017ImprovedGANs}
Ishaan Gulrajani, Faruk Ahmed, Martin Arjovsky, Vincent Dumoulin, and Aaron~C
  Courville.
\newblock {Improved Training of Wasserstein GANs}.
\newblock In I~Guyon, U~V Luxburg, S~Bengio, H~Wallach, R~Fergus,
  S~Vishwanathan, and R~Garnett, editors, \emph{Advances in Neural Information
  Processing Systems}, volume~30. Curran Associates, Inc., 2017.
\newblock URL
  \url{https://proceedings.neurips.cc/paper/2017/file/892c3b1c6dccd52936e27cbd0ff683d6-Paper.pdf}.

\bibitem[Durkan et~al.(2019)Durkan, Bekasov, Murray, and
  Papamakarios]{Durkan2019NeuralFlows}
Conor Durkan, Artur Bekasov, Iain Murray, and George Papamakarios.
\newblock {Neural spline flows}.
\newblock \emph{Advances in neural information processing systems}, 32, 2019.

\bibitem[Zhang et~al.(2017)Zhang, Cormode, Procopiuc, Srivastava, and
  Xiao]{Zhang2017Privbayes:Networks}
Jun Zhang, Graham Cormode, Cecilia~M Procopiuc, Divesh Srivastava, and Xiaokui
  Xiao.
\newblock {Privbayes: Private data release via bayesian networks}.
\newblock \emph{ACM Transactions on Database Systems (TODS)}, 42\penalty0
  (4):\penalty0 1--41, 2017.

\bibitem[Liu et~al.(2015)Liu, Luo, Wang, and Tang]{Liu2015DeepWild}
Ziwei Liu, Ping Luo, Xiaogang Wang, and Xiaoou Tang.
\newblock {Deep Learning Face Attributes in the Wild}.
\newblock In \emph{2015 IEEE International Conference on Computer Vision
  (ICCV)}, pages 3730--3738, 2015.
\newblock \doi{10.1109/ICCV.2015.425}.

\bibitem[Xu et~al.(2019{\natexlab{b}})Xu, Wu, Yuan, Zhang, and
  Wu]{xu2019achieving}
Depeng Xu, Yongkai Wu, Shuhan Yuan, Lu~Zhang, and Xintao Wu.
\newblock {Achieving causal fairness through generative adversarial networks}.
\newblock In \emph{IJCAI International Joint Conference on Artificial
  Intelligence}, volume 2019-August, pages 1452--1458. International Joint
  Conferences on Artificial Intelligence, 2019{\natexlab{b}}.
\newblock ISBN 9780999241141.
\newblock \doi{10.24963/ijcai.2019/201}.

\bibitem[van Breugel et~al.(2021)van Breugel, Kyono, Berrevoets, and van~der
  Schaar]{vanBreugel2021DECAF:Networks}
Boris van Breugel, Trent Kyono, Jeroen Berrevoets, and Mihaela van~der Schaar.
\newblock {DECAF: Generating Fair Synthetic Data Using Causally-Aware
  Generative Networks}.
\newblock In \emph{Advances in Neural Information Processing Systems}, 10 2021.
\newblock \doi{10.48550/arxiv.2110.12884}.
\newblock URL \url{https://arxiv.org/abs/2110.12884v2}.

\end{thebibliography}

\appendix
\onecolumn

\section{EXPERIMENTAL DETAILS} \label{appx:experimental_details}
\subsection{Workflow}

Input: Reference data $\mathcal{D}_{ref}$ , synthetic data $\mathcal{D}_{syn}$ and test data $\mathcal{D}_{test}$.\\
Output: $\hat{m}$ for all $x\in \mathcal{D}_{test}$.\\
Steps:
\begin{enumerate}
\item Train density model $p_R(X)$ on $\mathcal{D}_{ref}$.
\item Train density model $p_G(X)$ on $\mathcal{D}_{syn}$.
\item Compute $A_{DOMIAS}(x)=\frac{p_G(x)}{p_R(x)}$ for all $x\in\mathcal{D}_{test}$
\item Choose threshold $\tau$, e.g. $\tau = median \{A_{DOMIAS}(x)|x\in\mathcal{D}_{test}\}$
\item Infer 
\begin{equation*}
    \hat{m} =\begin{cases}1, &\text{if } A_{DOMIAS}(x)>\tau,\\ 0, &\text{ otherwise},
    \end{cases}
\end{equation*}
for all $x\in\cD_{test}$.
\end{enumerate}

\subsection{Data}
We use the California housing \citep{Pace1997SparseAutoregressions} (license: CC0 public domain) and Heart Failure (private) datasets, see Table \ref{tab:datasets} and Figure \ref{fig:data_correlation} for statistics. All data is standardised.

\begin{table}[hbt]
    \centering
    \caption{Dataset statistics}
    \label{tab:datasets}
    \begin{tabular}{l|c|c} \toprule
        & California Housing & Heart Failure \\ \midrule
        Number of samples & 20640 & 40300\\
        Number of features & 8 & 35\\
        - binary & 0 & 25\\
        - continuous & 8 & 10\\ \bottomrule
    \end{tabular}

\end{table}

\begin{figure*}[hbt]
    \centering
    \begin{subfigure}{0.48\textwidth}
    \centering
    \includegraphics[width=0.9\textwidth]{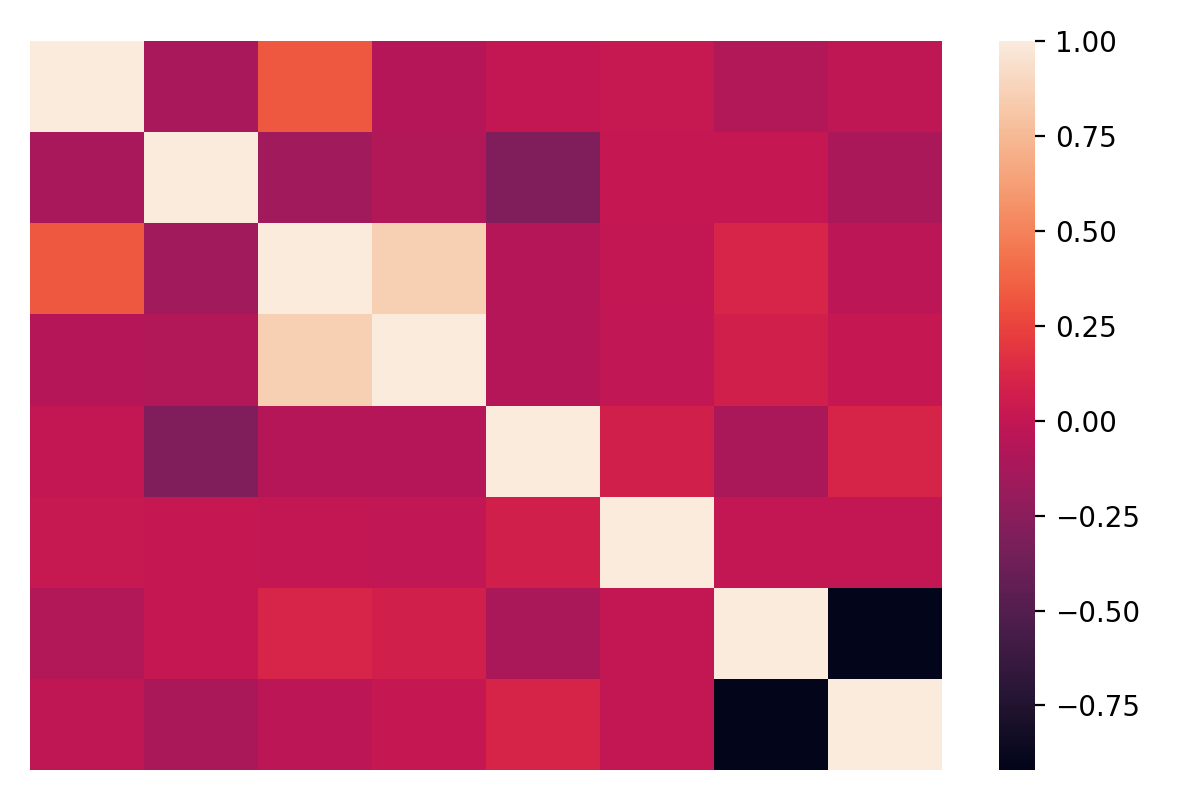}
    \caption{Housing}
    \end{subfigure}
    \hfill
    \begin{subfigure}{0.48\textwidth}
    \centering
    \includegraphics[width=0.9\textwidth]{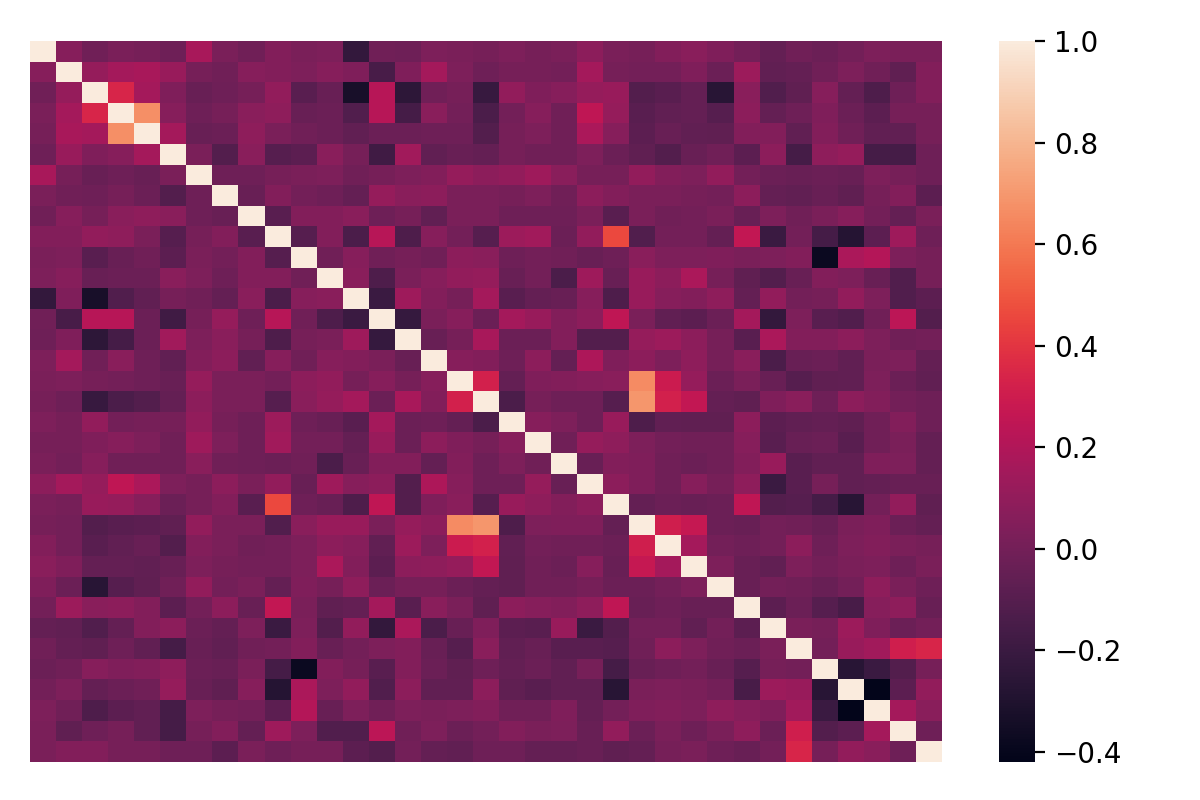}
    \caption{Heart Failure}
    \end{subfigure}
    \caption{Correlation matrices of features within Housing and Heart Failure datasets. The first feature of the Heart Failure dataset is used for defining the minority group in Section 5.3.}
    \label{fig:data_correlation}
\vspace{-0.25cm}\end{figure*}

\subsection{Experimental settings}

All results reported in our paper are based on $8$ repeated runs, with shaded area denoting standard deviations. We experiment on a machine with 8 Tesla K80 GPUs and 32 Intel(R) E5-2640 CPUs. We shuffle the dataset and split the dataset into training set, test set, and reference set. The attack performance is computed over a test set consisting of $50\%$ training data (i.e. samples from $\cD_{mem})$ and $50\%$ non-training data. Choices of sizes for those sets are elaborated below.

\paragraph{Experimental Details for Section 5.1}
In this section, we experimented on the California Housing Dataset to compare different MIA performance with DOMIAS. For the experiment varying the number of members in the training dataset (i.e. left panel of Figure 3), we use a fixed training epoch $2000$, a fixed number of reference example $|\cD_{ref}|=10000$ and a fixed number of generated example $|\cD_{syn}|=10000$. For the experiment varying the number of training epochs of TVAE (i.e. the right panel of Figure 3), we use a fixed training set size $|\cD_{mem}|=500$, a fixed number of reference example $|\cD_{ref}|=10000$ and a fixed number of generated example $|\cD_{syn}|=10000$. Training with a single seed takes $2$ hours to run in our machine with BNAF as the density estimator. 

In BNAF density estimation, the hyper-parameters we use are listed in Table~\ref{tab:hyper-param-bnaf}. Our implementation of TVAE is based on the source code provided by ~\citep{Xu2019ModelingGAN}.
\begin{table}[hbt]
    \centering
    \caption{Hyperparameters for BNAF}
    \label{tab:hyper-param-bnaf}
    \begin{tabular}{c|c}
    \toprule
        batch-dim & 50 \cr
        n-layer & 3 \\
        hidden-dim & 32 \\
        flows & 5 \\
        learning rate &$0.01$ \\
        epochs & 50 \\
    \bottomrule
    \end{tabular}
\end{table}

\paragraph{Experimental Details for Section 5.2}
In our experiments varying the number of reference data $n_{ref}$, i.e. results reported in the left panel of Figure 4, we fix the training epoch to be $2000$, set $n_{syn}=10000$ and $n_{M}=500$. In the experiments varying the number of generated data $n_{syn}$, i.e. results reported in the right panel of Figure 4, we set $n_{ref}=10000$, training epoch to be $2000$, and $n_{mem}=500$. Our implementation of the kernel density estimation is based on \textit{sklearn} with an automated adjusted bandwidth. Training with a single seed takes $0.5$ hours to finish in our machine with the kernel density estimator.

\paragraph{Experimental Details for Section 5.3} Based on results of Section 5.2, the attacking performance on different subgroups can be immediately calculated by adopting appropriate sample weights. 

\paragraph{Experimental Details for Section 5.4}
In the Additive-Noise baseline curve, results are generated with the following noise values: 
$[0.7,
0.9,
1.1,
1.3,
1.5,
1.7,
1.9,
2.3,
2.5,
2.9,
3.5,
3.9]$. In the ADS-GAN curve, results are generated with the following privacy parameter $\lambda = [0.2, 0.5, 0.7, 1.0, 1.1, 1.3, 1.5]$. 
In the WGAN-GP we use a gradient penalty coefficient $10.0$. All the other methods are implemented with recommended hyper-parameter settings. Training different generative models are not computational expensive and take no more than $10$ minutes to finish in our machine. Using a kernel density estimator and evaluating all baseline methods take another $20$ minutes, while using a BNAF estimator takes around $1.5$ more hours.



\section{ADDITIONAL EXPERIMENTS} \label{appx:additional_results}

\subsection{Experiment 5.1 and 5.2 on Heart Failure dataset}
We repeat the experiments of Section 5.1 and 5.2 on the Heart Failure dataset, see Figures \ref{fig:domias_vs_baselines_heart_failure} and \ref{fig:ablation_results_heart_failure}. Results are noisier, but we observe the same trends as in Sections 5.1 and 5.2

\begin{figure*}[hbt]

    \centering
    \begin{subfigure}{0.48\textwidth}
    \centering
    \includegraphics[width=0.9\textwidth]{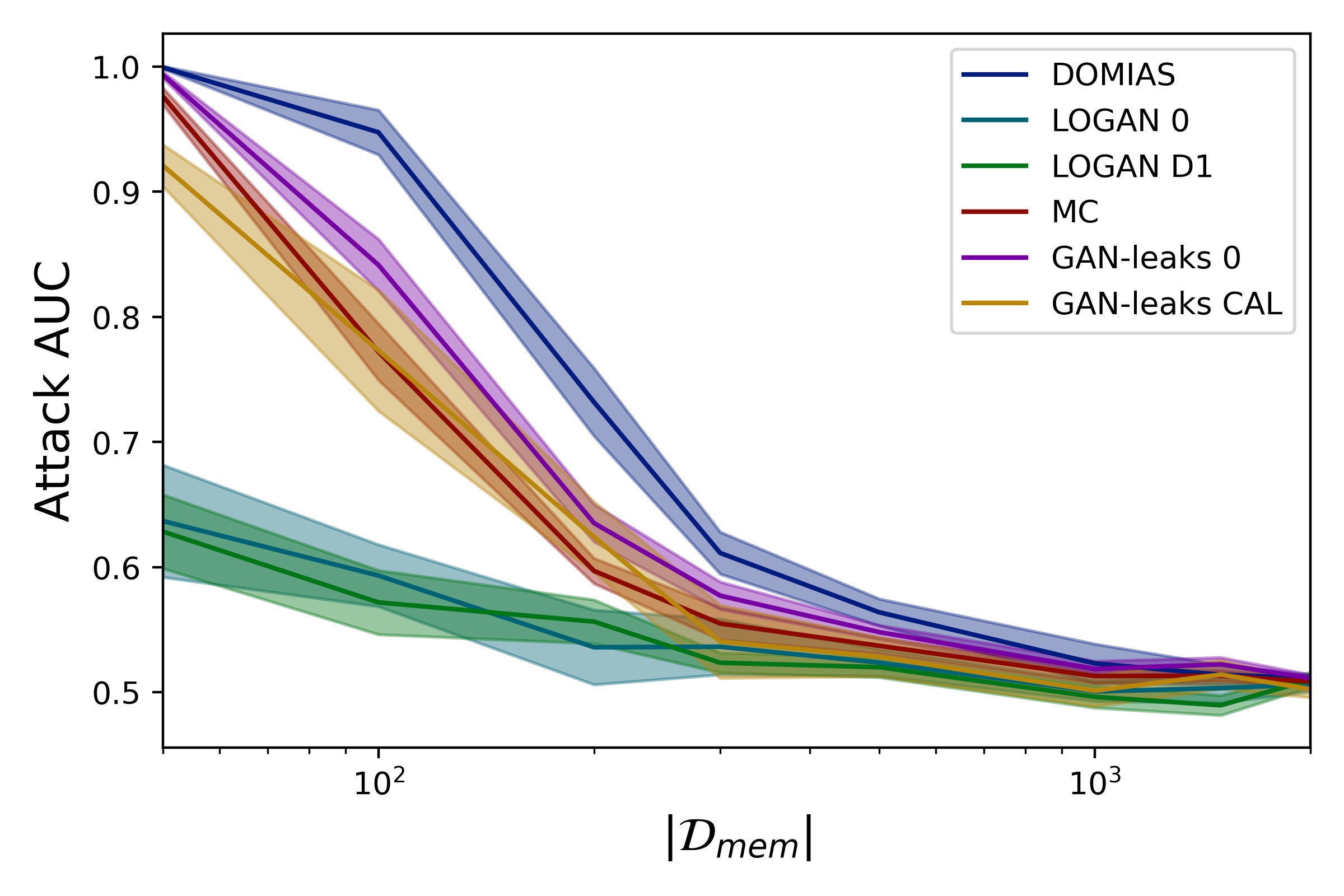}
    \end{subfigure}
    \hfill
    \begin{subfigure}{0.48\textwidth}
    \centering
    \includegraphics[width=0.9\textwidth]{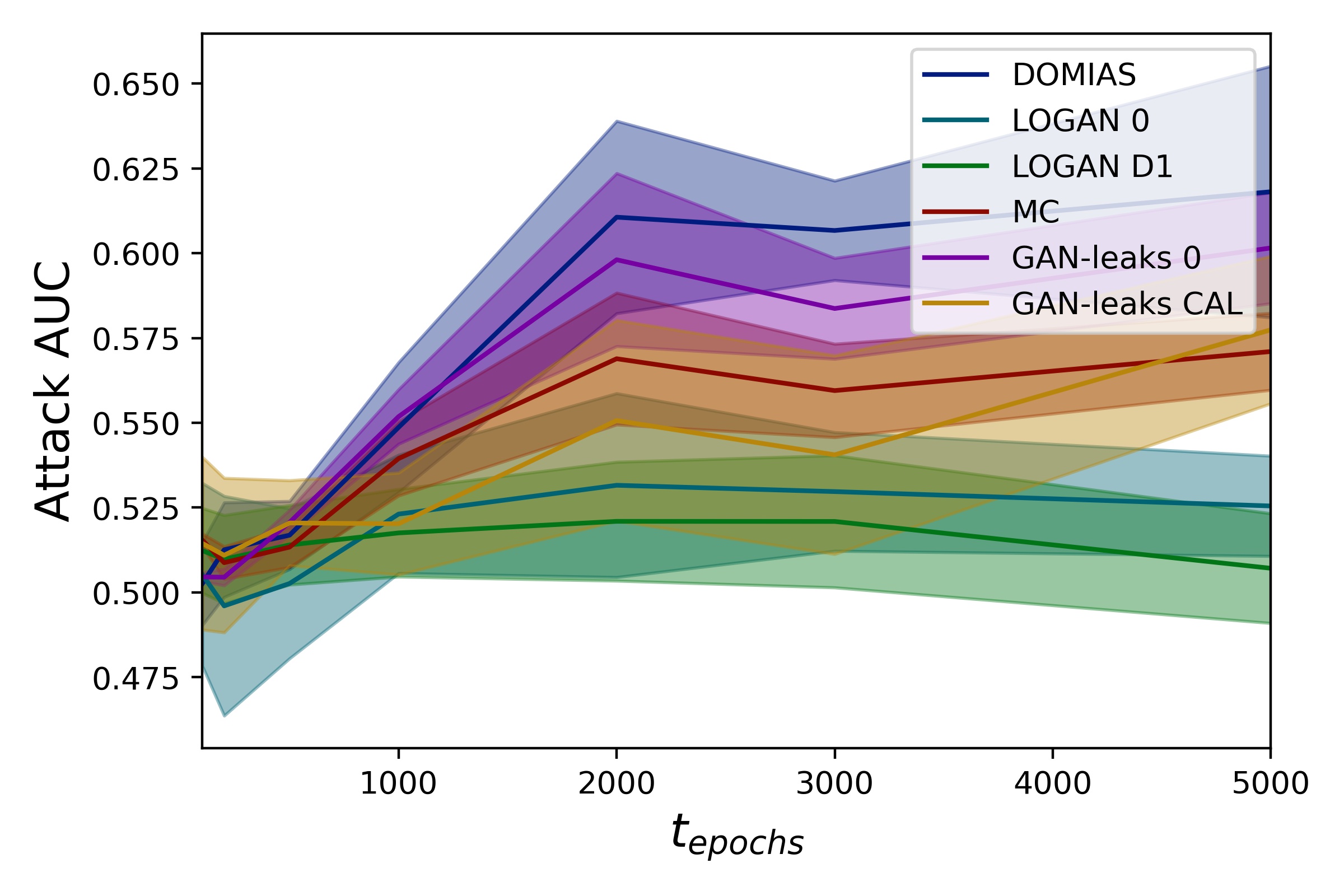}
    \end{subfigure}
    \caption{\textit{DOMIAS outperforms baselines on Heart Failure dataset.} MIA performance of DOMIAS and baselines versus the generative model training set size $|\cD_{mem}|$ and training time $t_{epochs}$, evaluated on Heart Failure datasets. The same trends are observed as in Section 5.1.}
    \label{fig:domias_vs_baselines_heart_failure}
\vspace{-0.25cm}
\end{figure*}

\begin{figure*}[hbt]
    \centering
    \begin{subfigure}{0.48\textwidth}
    \includegraphics[width=0.9\textwidth]{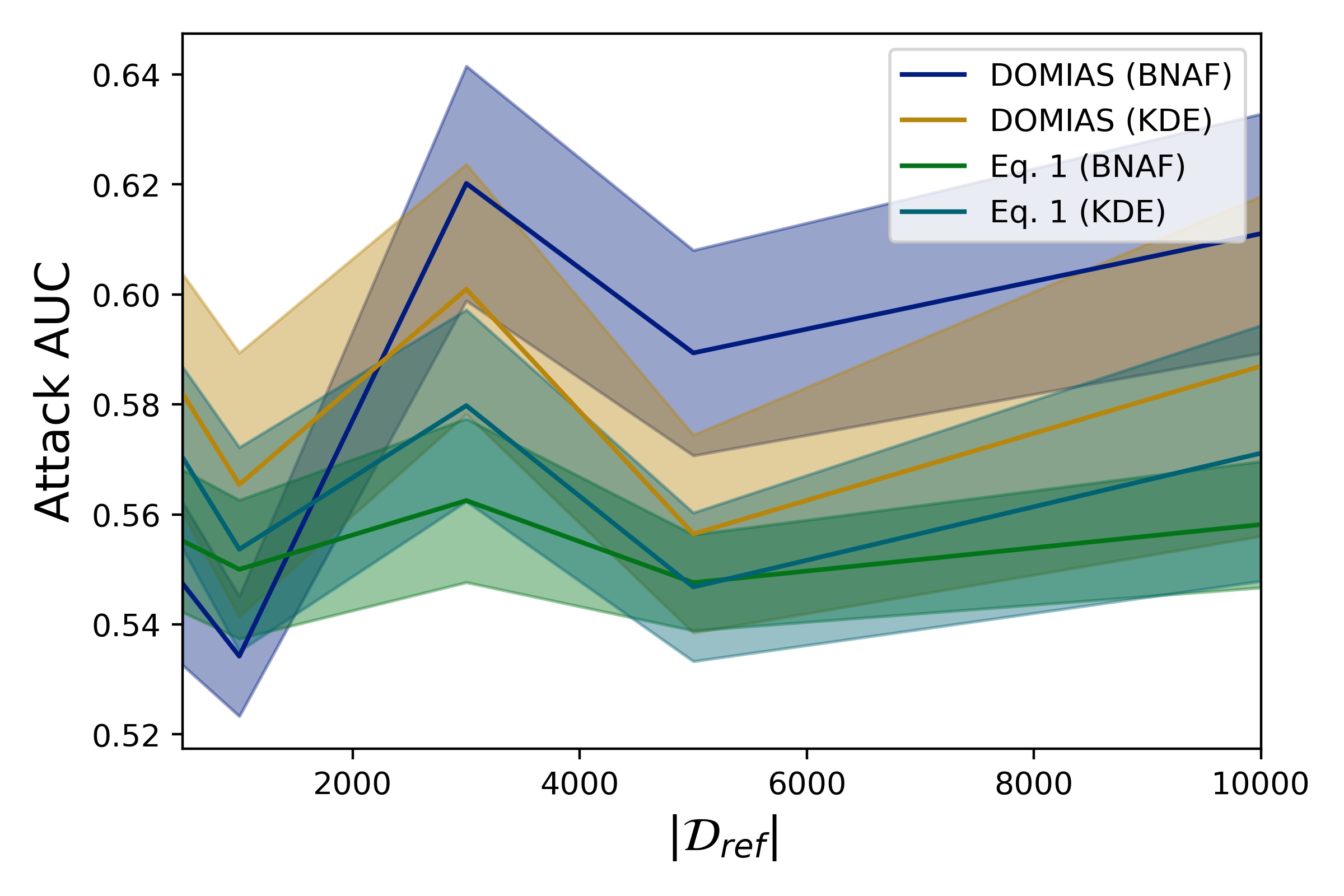}
    \end{subfigure}
    \hfill
    \begin{subfigure}{0.48\textwidth}
    \centering
    \includegraphics[width=0.9\textwidth]{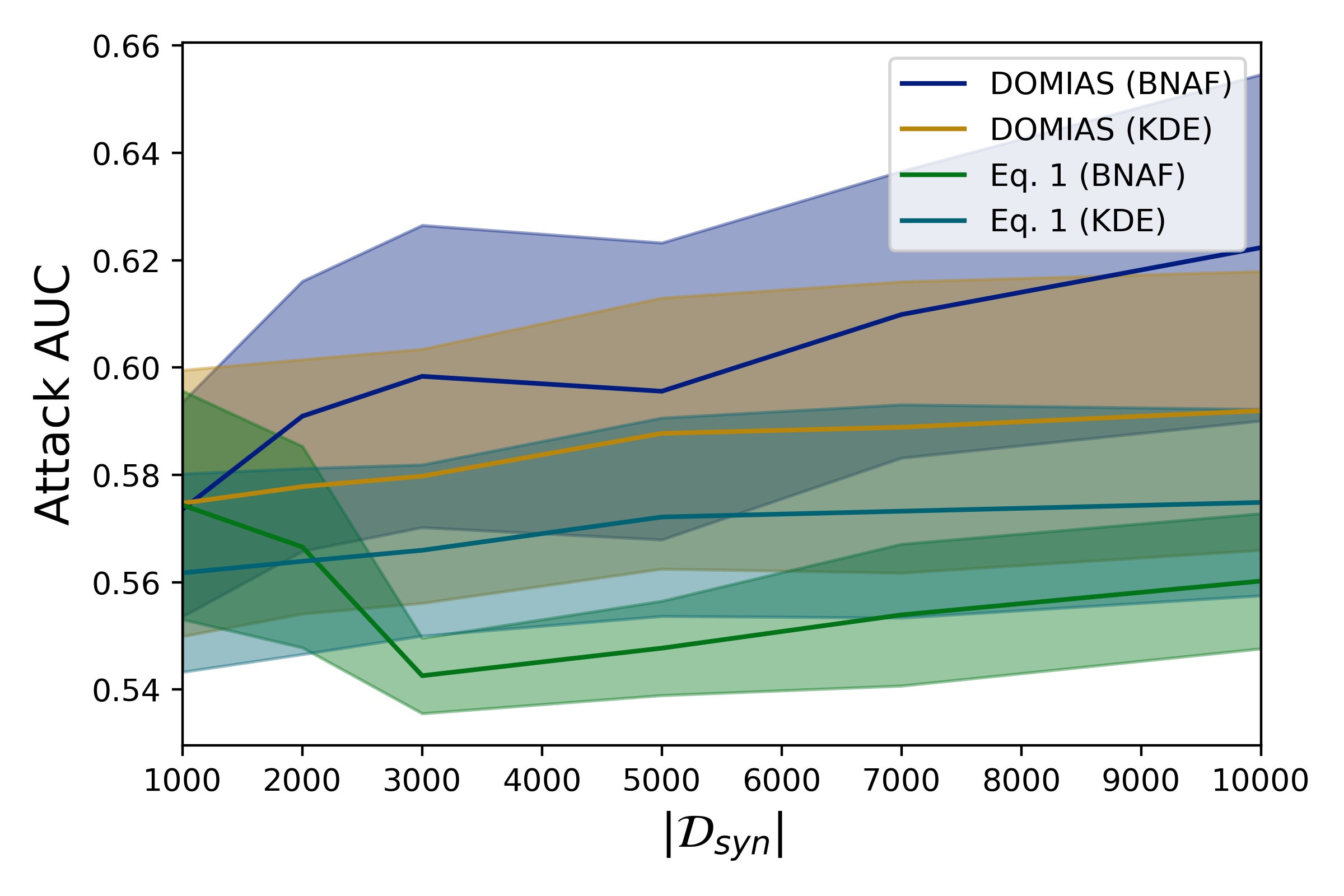}
    \end{subfigure}
    \caption{\emph{DOMIAS source of gain.} Ablation study of DOMIAS on Heart Failure dataset, with attack performance as a function of the reference dataset size (left) and the synthetic dataset size (right). Similar to Section 5.2,  we see that the MIA performance of DOMIAS is largely due to assumption Eq.2 vs Eq. 1, i.e. the value of the reference dataset.}
    \label{fig:ablation_results_heart_failure}
\vspace{-0.25cm}
\end{figure*}

\subsection{Experiment 5.4: Results other attackers}
In Figure \ref{fig:other_attackers_against_gen} we include the results of experiment 5.4 for all attacks, including error bars. Indeed, we see that DOMIAS outperforms all baselines against most generative models. This motivates using DOMIAS for quantifying worst-case MIA vulnerability.

\begin{figure*}[hbt]
    \centering
    \includegraphics[width=\textwidth]{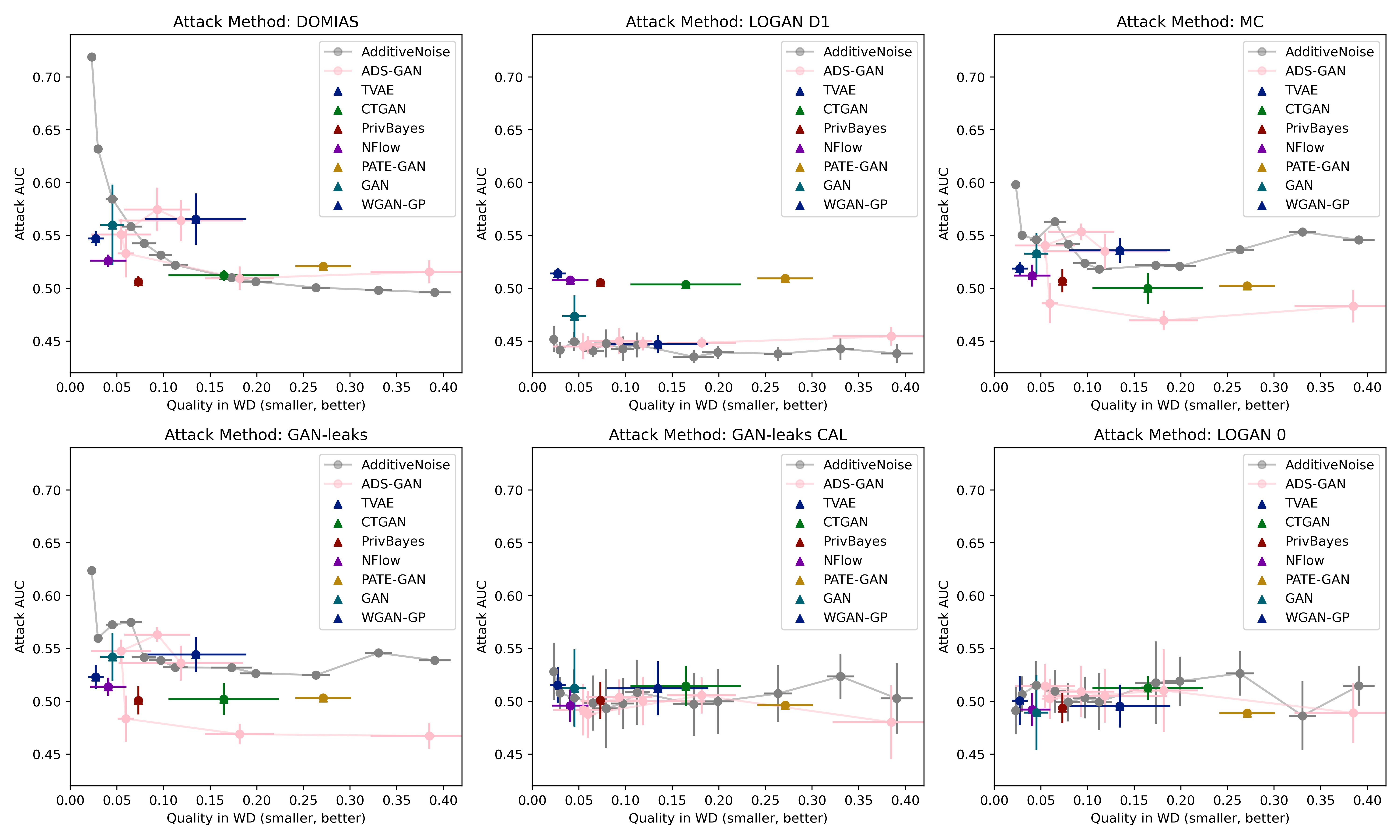}
    \caption{DOMIAS consistently outperforms baseline attackers at attacking the different generative models.}
    \label{fig:other_attackers_against_gen}
\end{figure*}

\subsection{CelebA image data} \label{sec:celeba}
We include additional results for membership inference attacks against the image dataset CelebA. Results indicate DOMIAS is significantly better at attacking this high-dimensional data than baseline methods.

\paragraph{Set-up} We use \href{https://mmlab.ie.cuhk.edu.hk/projects/CelebA.html}{CelebA}~\citep{Liu2015DeepWild}, a large-scale face attributes dataset with more than 200K celebrity images. We generate a synthetic dataset with 10k examples using a convolutional VAE with a training set containing the first 1k examples, and use the following 1k examples as test set. Then the following 10k examples are used as reference dataset. As training the BNAF density estimator is computational expensive (especially when using deeper models), we conduct dimensionality reduction with a convolutional auto-encoder with $128$ hidden units in the latent representation space (i.e. output of the encoder) and apply BNAF in such a representation space. The hyper-parameters and network details we use in VAE are listed in Table~\ref{tab:hyper-param-vae} and Table \ref{tab:vae_architecture}.

\begin{table}[hbt]
    \centering
    
    \caption{Hyperparameters for VAE}
    \label{tab:hyper-param-vae}
    \begin{tabular}{c|c}
    \toprule
        batch size & 128 \cr
        n-layer & 5 \cr
        Optimizer & Adam \cr
        learning rate &$0.002$ \cr
    \bottomrule
    \end{tabular}
\end{table}

\begin{table*}
    \centering
    \caption{Architecture of VAE}
    \label{tab:vae_architecture}
    \begin{subtable}[t]{0.48\textwidth}
        \caption{Network Structure for Encoder}
    \label{tab:stru-encoder}
    \begin{tabular}{c|c}
    \toprule
        Layer & Params (PyTorch-Style) \cr
        \hline
        Conv1 & $(3,64,4,2,1)$ \cr
        ReLU & $\cdot$\cr
        Conv2 & $(64,128,4,2,1)$ \cr
        ReLU & $\cdot$\cr
        Conv3 & $(128,256,4,2,1)$ \cr
        ReLU & $\cdot$\cr
        Conv4 & $(256,256,4,2,1)$ \cr
        ReLU &$\cdot$ \cr
        Linear1 & $(256*4*4,256)$ \cr
        ReLU & $\cdot$\cr
        Linear2 & $(256,256)$ \cr
        ReLU & $\cdot$\cr
        Linear3 & $(256,128*2)$ \cr
    \bottomrule
    \end{tabular}
    \end{subtable}
    \hspace{\fill}
    \begin{subtable}[t]{0.48\textwidth}
    \caption{Network Structure for Decoder}
    \label{tab:stru-decoder}
    \begin{tabular}{c|c}
    \toprule
        Layer & Params (PyTorch-Style) \cr
        \hline
        Linear1 & $(128,256)$ \cr
        ReLU & $\cdot$\cr
        Linear2 & $(256,256)$ \cr
        ReLU & $\cdot$\cr
        Linear3 & $(256,256*4*4)$ \cr
        ReLU & $\cdot$\cr
        ConvTranspose1 & $(256,256,4,2,1)$ \cr
        ReLU & $\cdot$\cr
        ConvTranspose2 & $(256,128,4,2,1)$ \cr
        ReLU & $\cdot$\cr
        ConvTranspose3 & $(128,64,4,2,1)$ \cr
        ReLU & $\cdot$\cr
        ConvTranspose4 & $(64,3,4,2,1)$ \cr
        Tanh & $\cdot$\cr
    \bottomrule
    \end{tabular}
    \end{subtable}
\end{table*}

\paragraph{Results}
Figure \ref{fig:celeba} includes the attacking AUC of DOMIAS and baselines of 8 runs. DOMIAS consistently outperforms other MIA methods, most of which score not much better than random guessing. These methods fail to attack the 128-dimensional representations of the data (originally $64\times 64$ pixel images), due to most of them using nearest neighbour or KDE-based approaches. On the other hand, DOMIAS is based on the flow-based density estimator BNAF \citep{deCao2019BlockFlow}, which is a deeper model that is more apt at handling the high-dimensional data.
\begin{figure*}[hbt]
    \centering
    \includegraphics[width=0.9\textwidth]{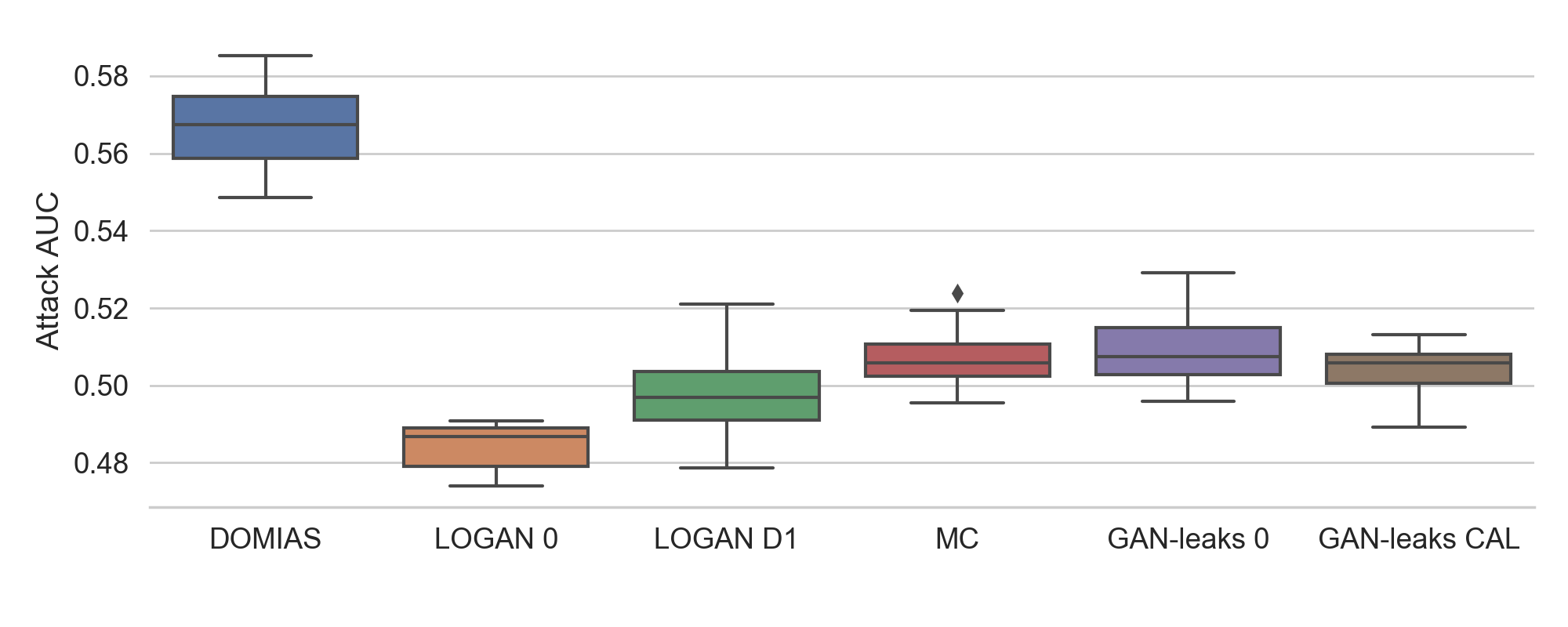}
    \caption{\emph{Attacking performance on CelebA.} DOMIAS scores significantly better at attacking image data compared to baselines.}
    \label{fig:celeba}
\end{figure*}

\section{HIGH-LEVEL PRIOR KNOWLEDGE} \label{appx:gaussian_prior}
If we have no reference data at all, we can still perform more successful attacks compared to baselines if we have high-level statistics of the underlying distribution. Effectively, any informed prior can improve upon methods that use Eq. \ref{eq:assumption_prev}; this being a special case of Eq. \ref{eq:assumption_domias}, where one assumes a uniform prior on $p_R$. In this Appendix, we use the Housing dataset and we assume that we only know the mean and standard deviation of the first variable, median income. This is a very realistic setting in practice, since an adversary can relatively easily acquire population statistics for individual features. We subsequently model the reference dataset distribution $p_{ref}$ as a normal distribution of only the age higher-level statistics---i.e. not making any assumptions on any of the other variables, implicitly putting a uniform prior on these when modelling $p_{ref}$. Otherwise, we use the same training settings as in Experiment 5.1 (left panel Figure 3). In Figure \ref{fig:high-level statistics}. We see that even with this minimal assumption, we still outperform its ablated versions. These results indicate that a relatively weak prior on the underlying distribution without any reference data, can still provide a relatively good attacker model.

\begin{figure}[hbt]
    \centering
    \includegraphics[width=0.8\columnwidth]{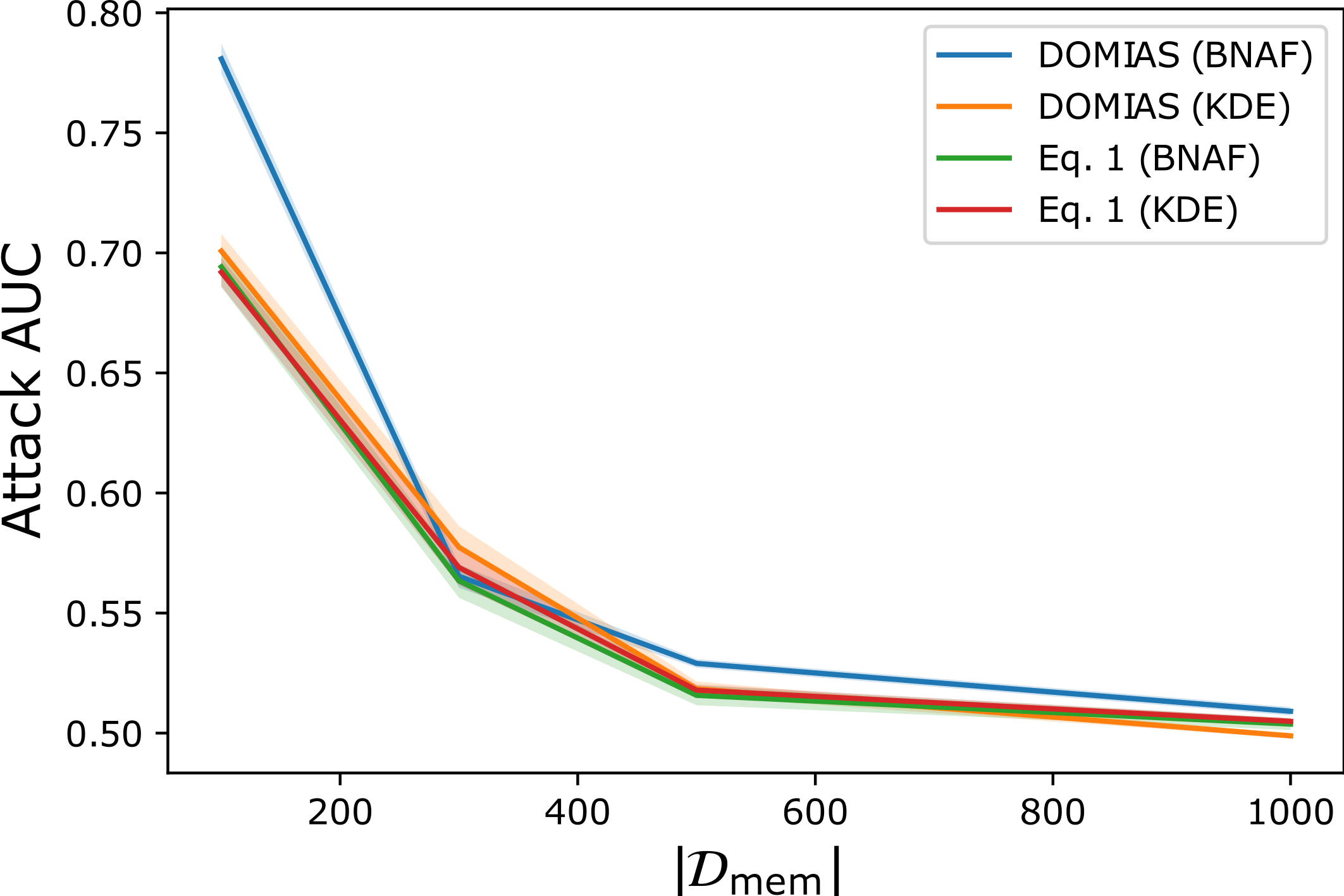}
    \caption{\textit{Using DOMIAS with no reference data but high-level statistics of the underlying data.} Using just the mean and standard deviation of the population's median income, DOMIAS outperforms its ablated counterparts that are based on Eq. \ref{eq:assumption_prev}. }
    \label{fig:high-level statistics}
\end{figure}

\section{HIGH-PRECISION ATTACKS} \label{appx:high_precision attacks}
\citet{Hu2021MembershipRegions} focus on high-precision membership attacks, i.e. can we attack a small set of samples with high certainty. This is an interesting question, since the risk of high-precision attacks may be hidden if one only looks at overall attacking performance. Their work is not applicable to our setting, e.g. they assume full generator and discriminator access. In this section, we show that even in the full black-box setting high-precision MIAs are a serious risk.

\subsection{Tabular data}

\paragraph{Set-up} We assume the same dataset and generative model set-up as in Section 5.3. We study which samples the different methods give the highest score, i.e. mark as most likely to be in $\cD_{mem}$. Let $\cD_{test}$ be a test set consisting for 50\%  of samples $x^i$ in $\cD_{mem}$ and 50\% samples not in $\cD_{mem}$, respectively denoted by $m=1$ and $m=0$. Let $\hat{m} = A(x)$ be the attacker's prediction, and let $S(A, \cD_{test},q) = \{x\in\cD_{test}|\hat{m}>Quantile(\{\hat{m}^i\}_i,1-q)\}$ be the set of samples that are given the $q$-quantile's highest score by attacker $A$. We are interested in the mean membership of this set, i.e. the precision if threshold $Quantile(\{\hat{m}^i|x^i\in\cD_{test}\},1-q)$ is chosen. We include results for DOMIAS and all baselines. Results are averaged over 8 runs.

\paragraph{Results} In Figure \ref{fig:high-precision} we plot the top-score precision-quantile curve for each method for each MIA method, i.e. $P(A, \cD_{test}, q) = \text{mean}(\{m|x\in S(A, \cD_{test}, q)\})$ as a function of $q$. These figures show the accuracy of a high-precision attacker, if this attacker would choose to attack only the top $q$-quantile of samples. We see that unlike other methods, the precision of DOMIAS goes down almost linearly and more gradually. Though MC and GAN-Leaks are able to find the most overfitted examples, they do not find all---resulting from their flawed underlying assumption Eq. 1 that prohibits them from finding overfitted examples in low-density regions.

\begin{figure*}[hbt]
    \centering
    \begin{subfigure}{0.33\textwidth}
    \centering
    \includegraphics[width=\textwidth]{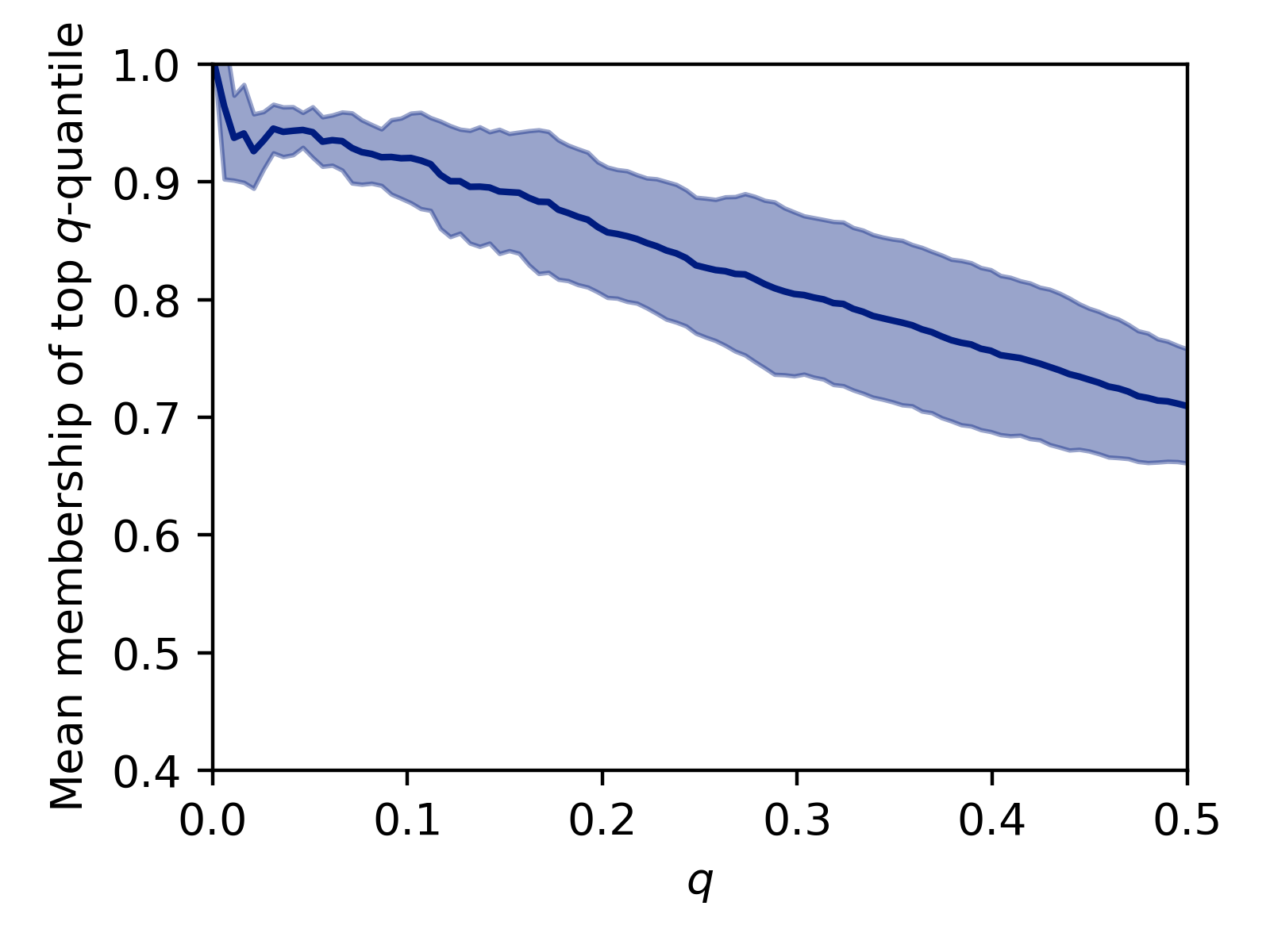}
    \caption{DOMIAS}
    \end{subfigure}
    \begin{subfigure}{0.33\textwidth}
    \centering
    \includegraphics[width=\textwidth]{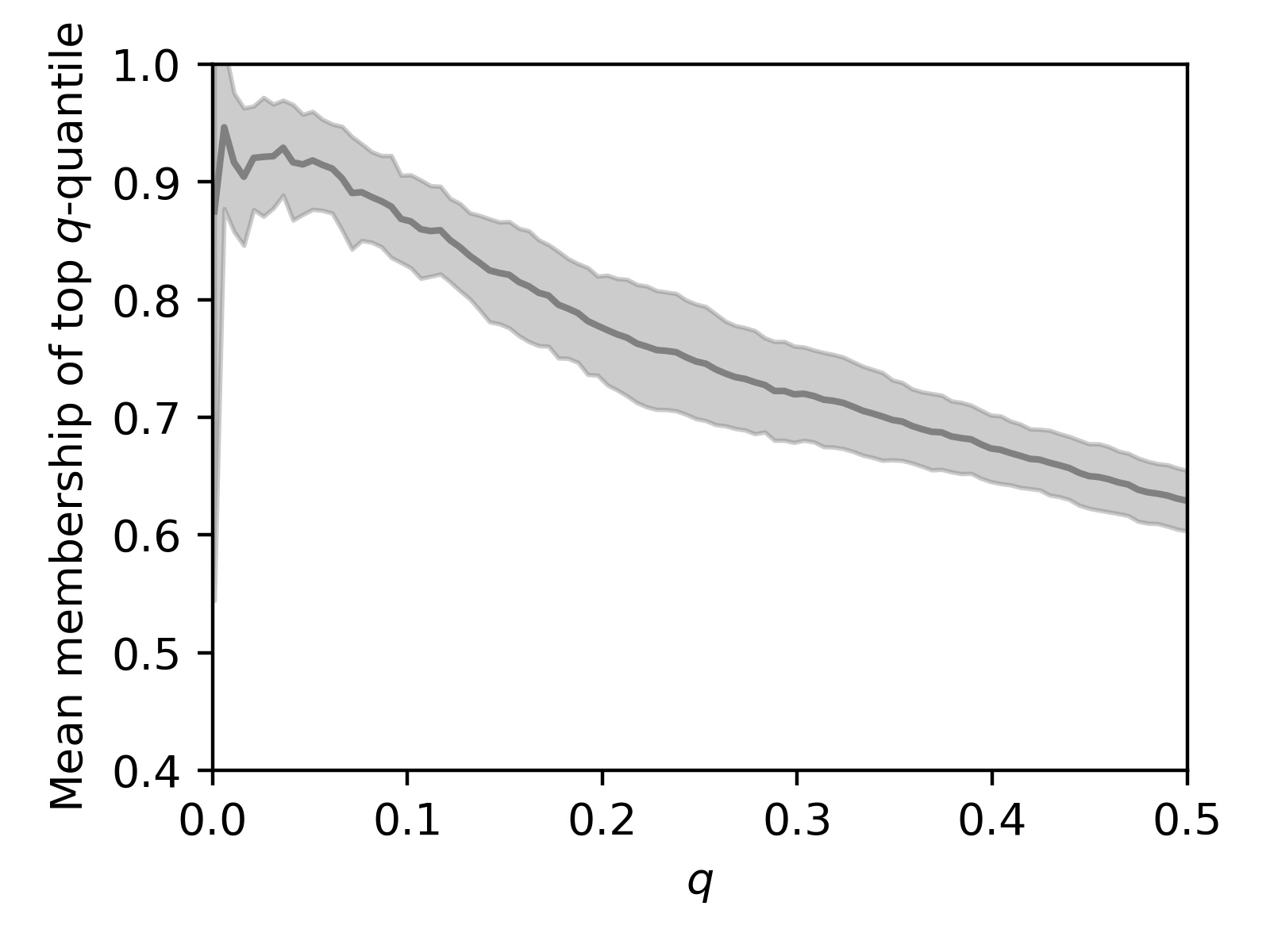}
    \caption{Eq. 1 (BNAF)}
    \end{subfigure}
    \begin{subfigure}{0.33\textwidth}
    \centering
    \includegraphics[width=\textwidth]{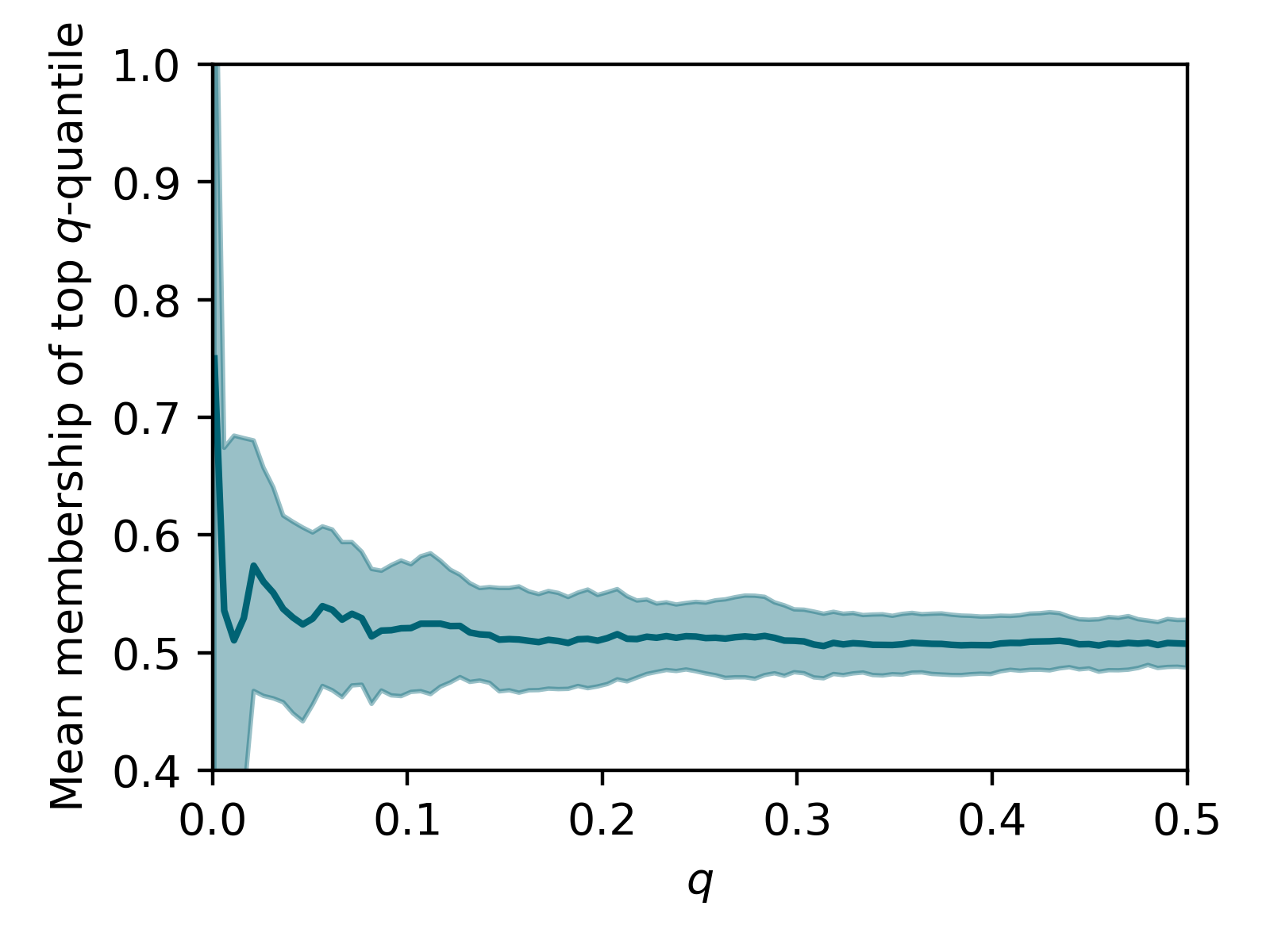}
    \caption{LOGAN 0}
    \end{subfigure}
    \begin{subfigure}{0.33\textwidth}
    \centering
    \includegraphics[width=\textwidth]{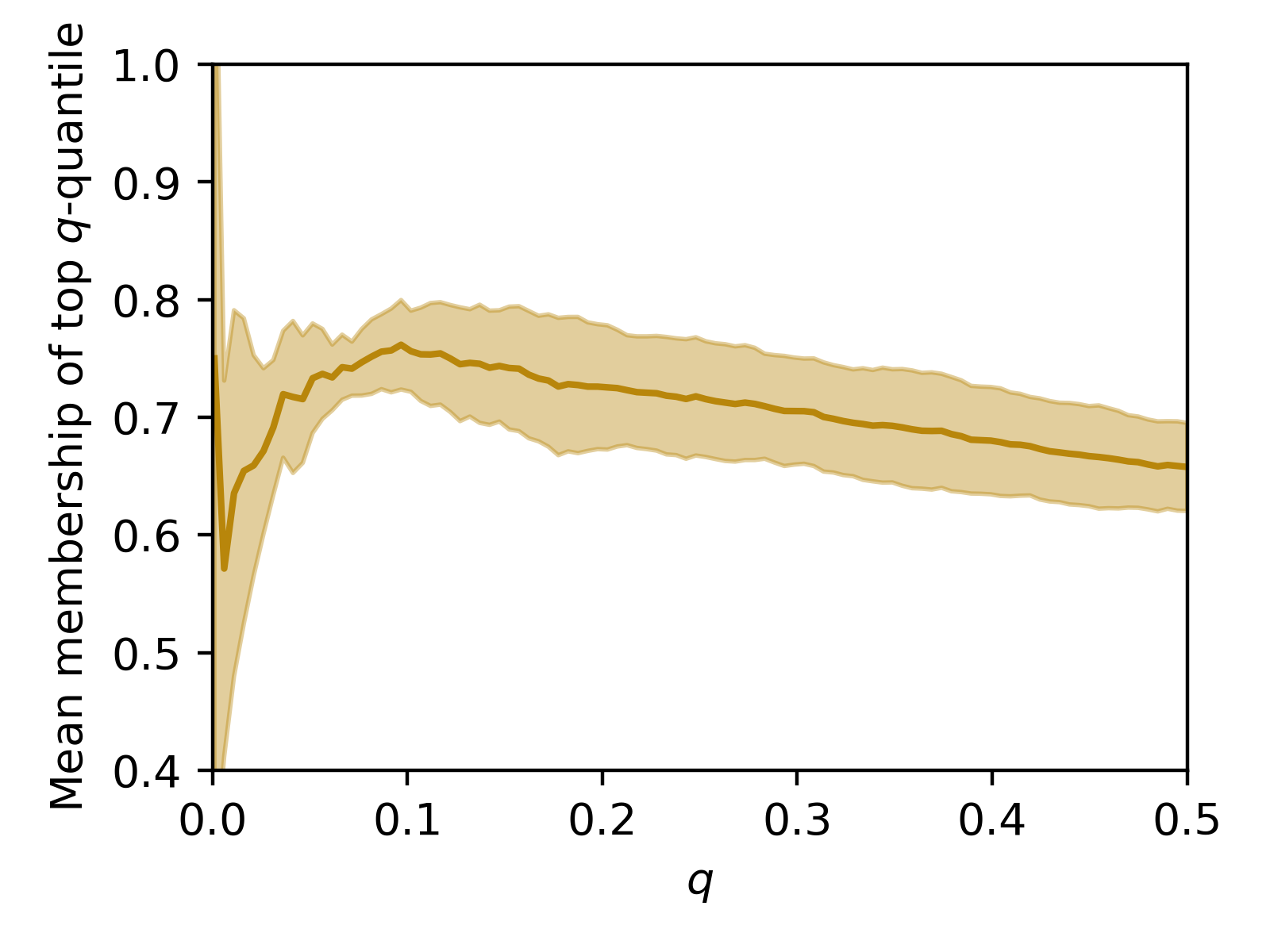}
    \caption{LOGAN D1}
    \end{subfigure} 
    \begin{subfigure}{0.33\textwidth}
    \centering
    \includegraphics[width=\textwidth]{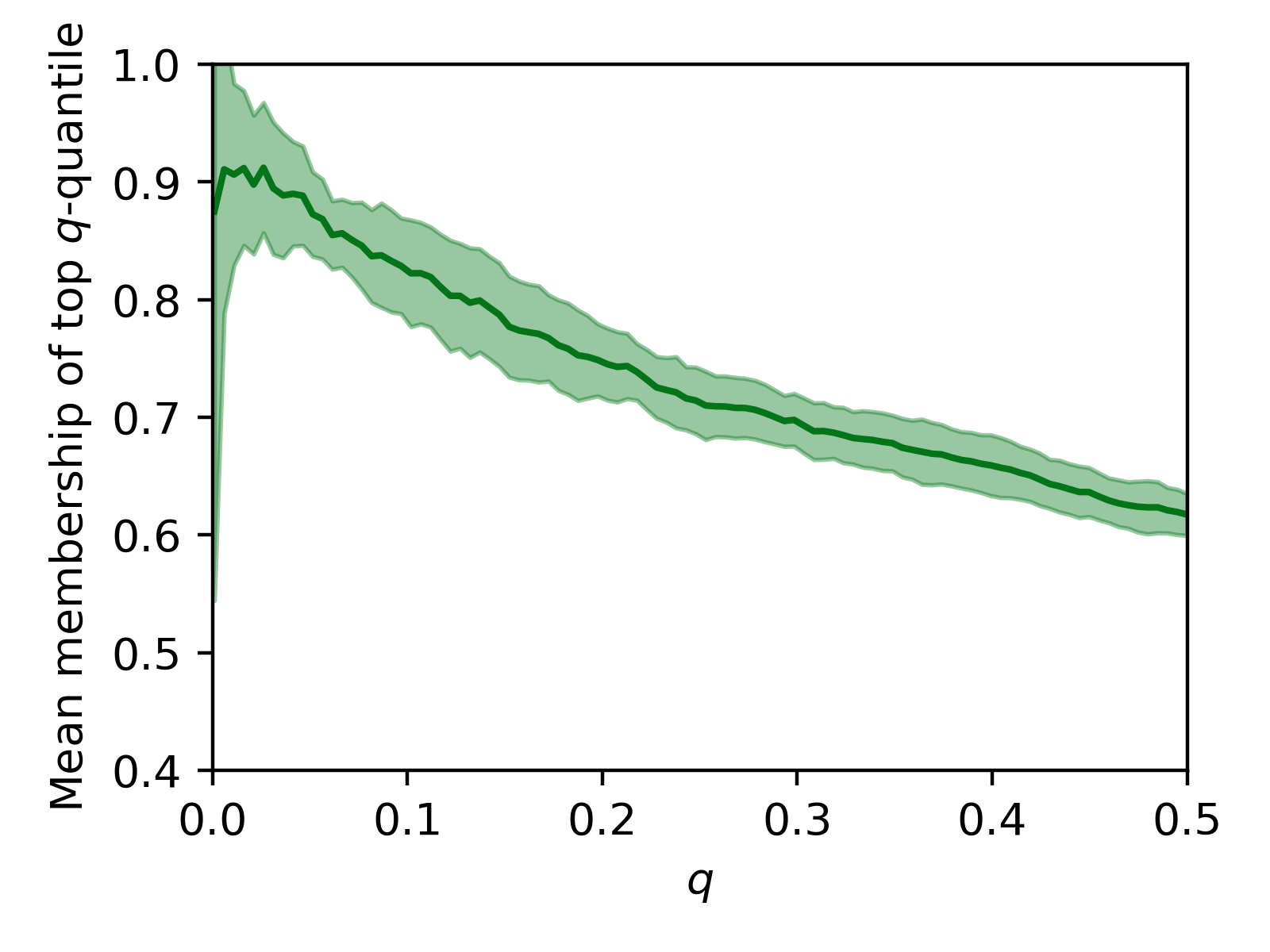}
    \caption{GAN-leaks 0}
    \end{subfigure}
    \begin{subfigure}{0.33\textwidth}
    \centering
    \includegraphics[width=\textwidth]{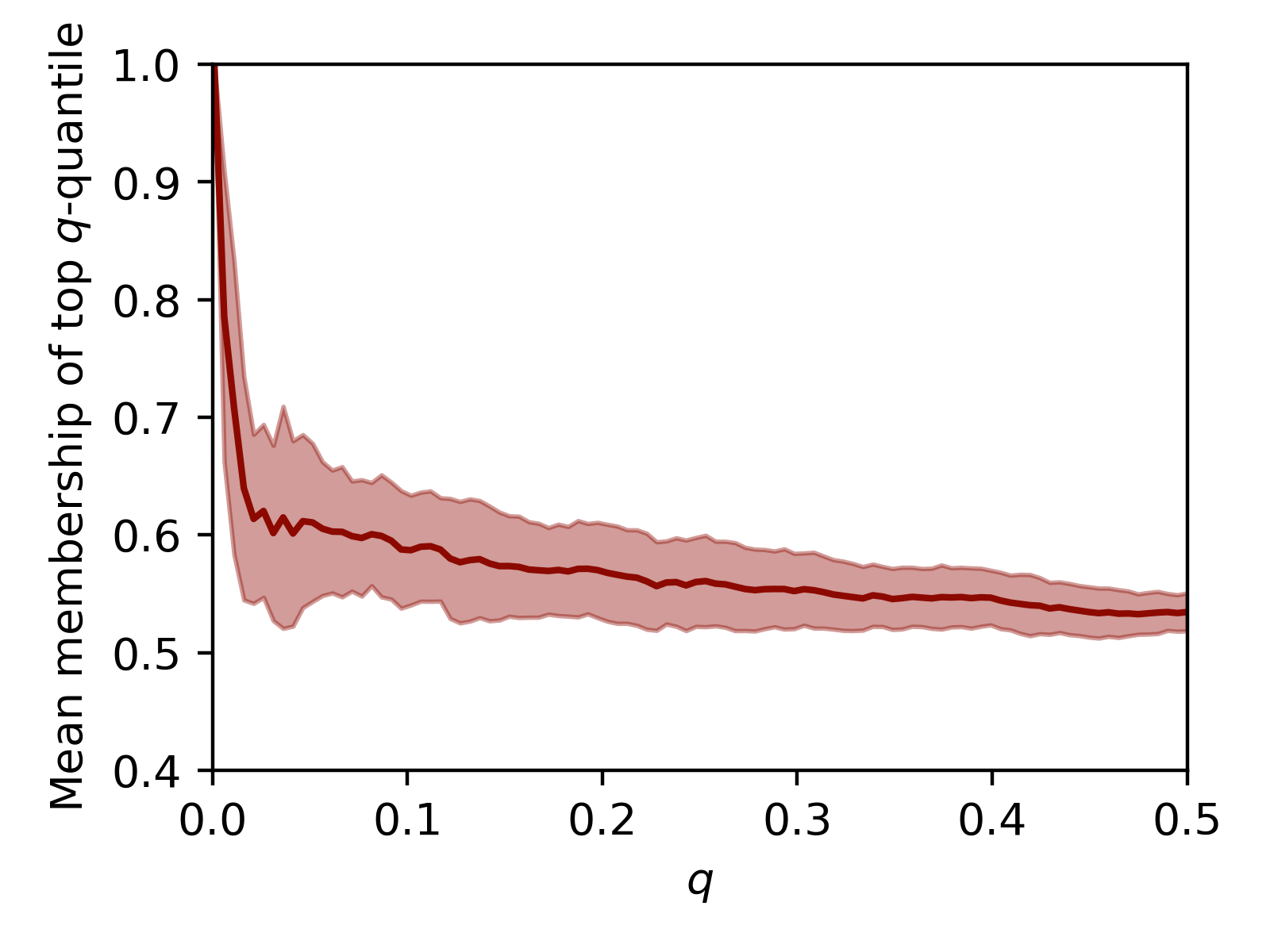}
    \caption{GAN-leaks CAL}
    \end{subfigure}
    \begin{subfigure}{0.33\textwidth}
    \centering
    \includegraphics[width=\textwidth]{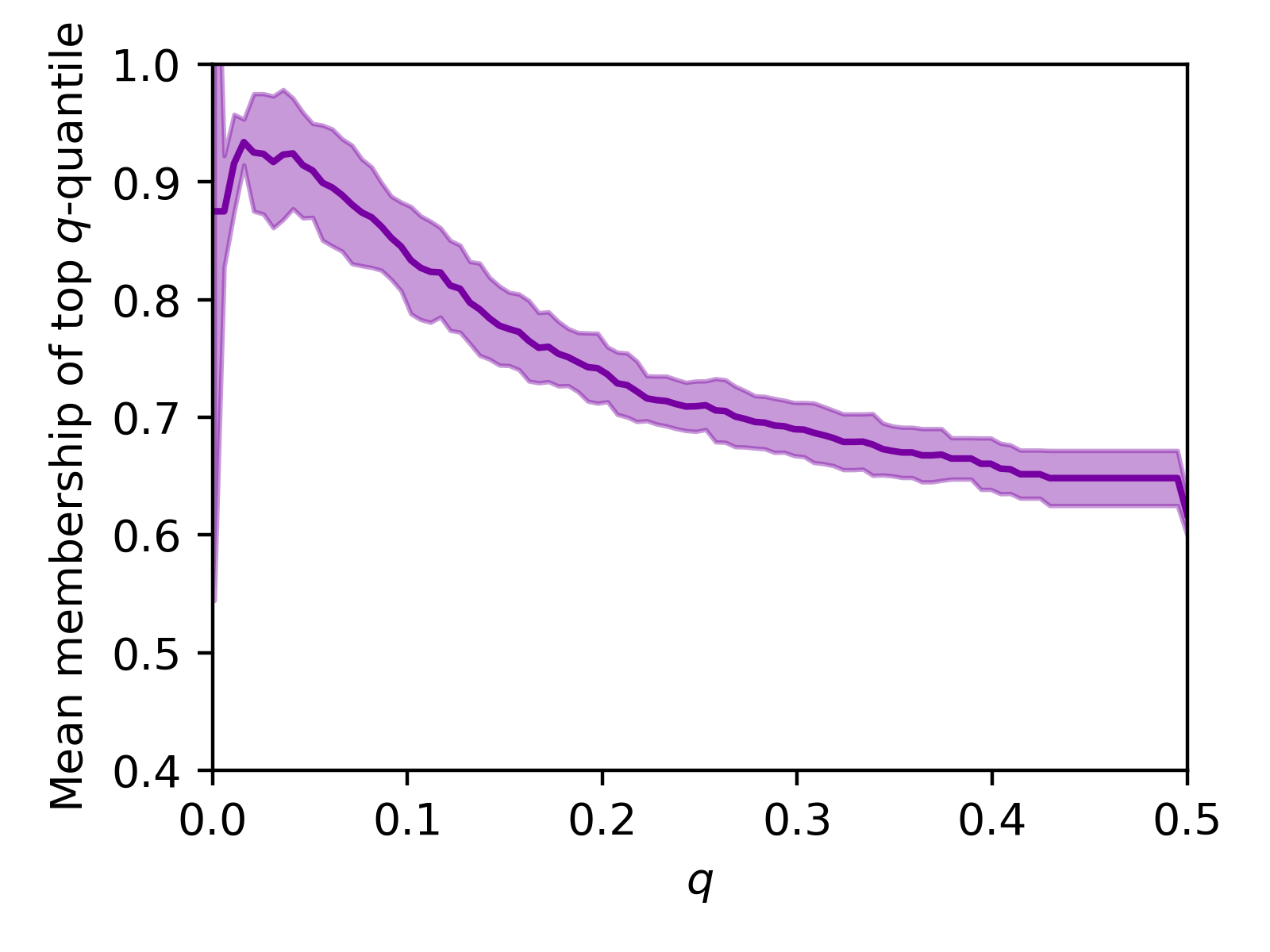}
    \caption{MC}
    \end{subfigure}
    \caption{\emph{DOMIAS is better at high-precision attacks than baselines on heart failure dataset.} Plotting the top-quantile precision $P(A, \cD_{test}, q)$ versus $q$. For example, if the attacker decides to attack only the $20\%$ highest samples, we get DOMIAS is significantly more precise ($86.2\pm 5.5\%$) compared to baselines---LOGAN D0 ($51.0\pm 3.9\%$), LOGAN D1 ($72.6\pm 5.3\%$), MC ($74.2\pm 3.0\%$), GAN-leaks ($74.9\pm 3.1\%$), GAN-Leaks CAL ($57.0\pm 4.1\%$). Additionally included is Eq. 1 (BNAF), the ablation attacker that does not make use of the reference data. We see that the reference data helps DOMIAS attack a a larger group with high precision.}
    \label{fig:high-precision}
\end{figure*}

\subsection{Image data}
Let us run the same high-precision attack on the CelebA dataset---see Appendix \ref{sec:celeba}, including settings. Again, we see that high-precision attacks are more successful when using DOMIAS, see Figure \ref{fig:high-precision-celeba}

\begin{figure*}[hbt]
    \centering
    \begin{subfigure}{0.33\textwidth}
    \centering
    \includegraphics[width=\textwidth]{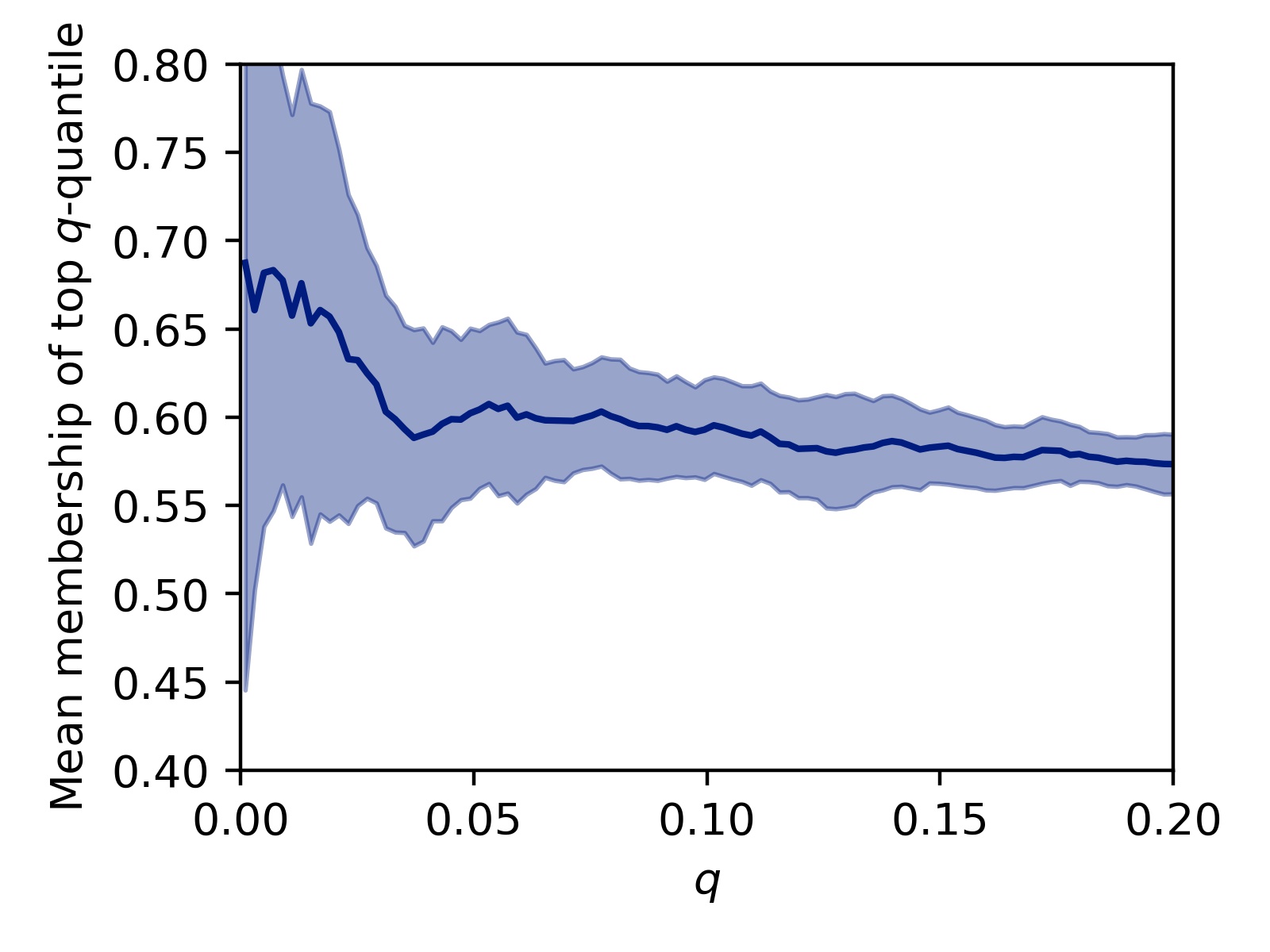}
    \caption{DOMIAS}
    \end{subfigure}
    \begin{subfigure}{0.33\textwidth}
    \centering
    \includegraphics[width=\textwidth]{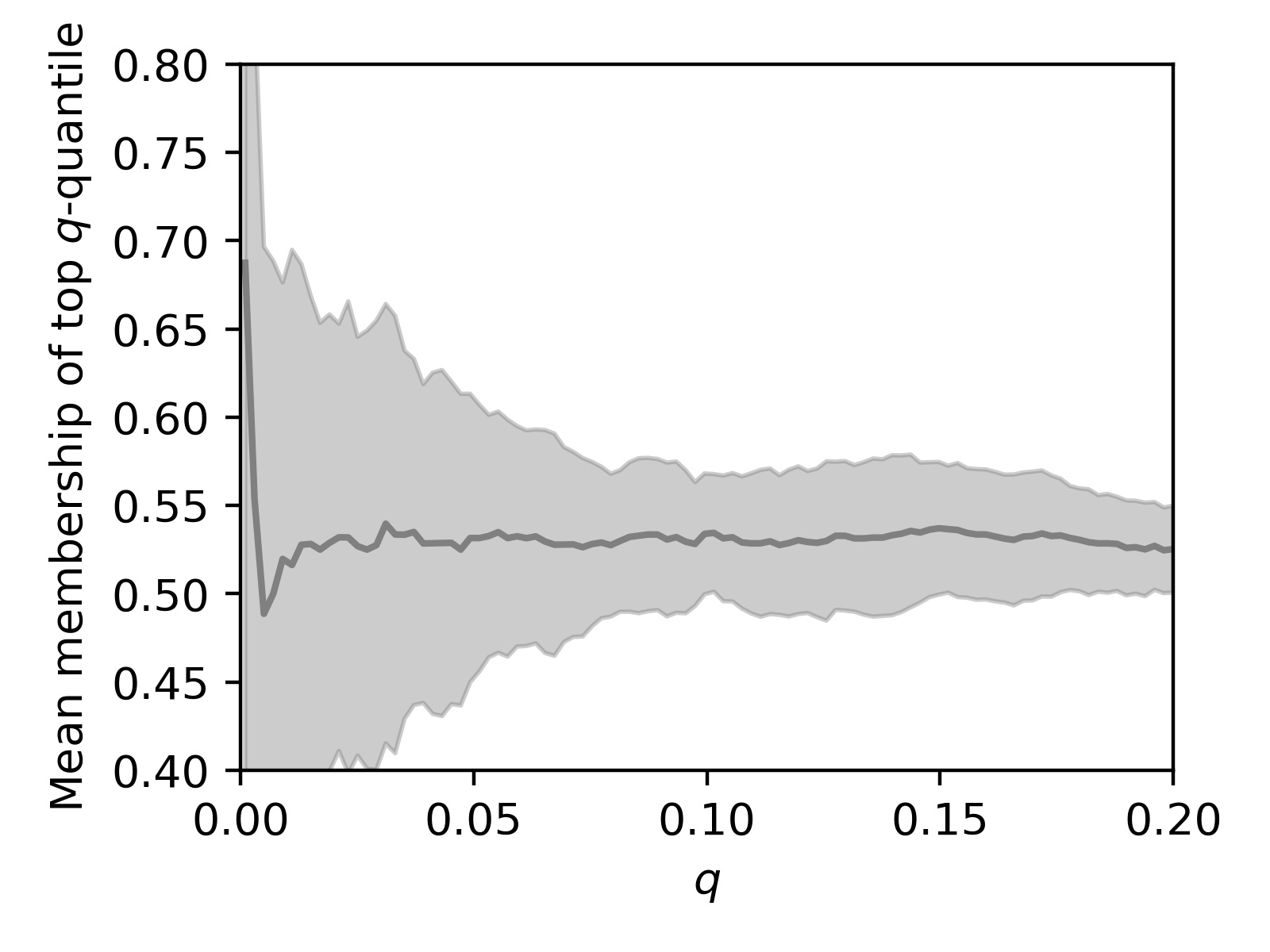}
    \caption{Eq. 1 (BNAF)}
    \end{subfigure}
    \begin{subfigure}{0.33\textwidth}
    \centering
    \includegraphics[width=\textwidth]{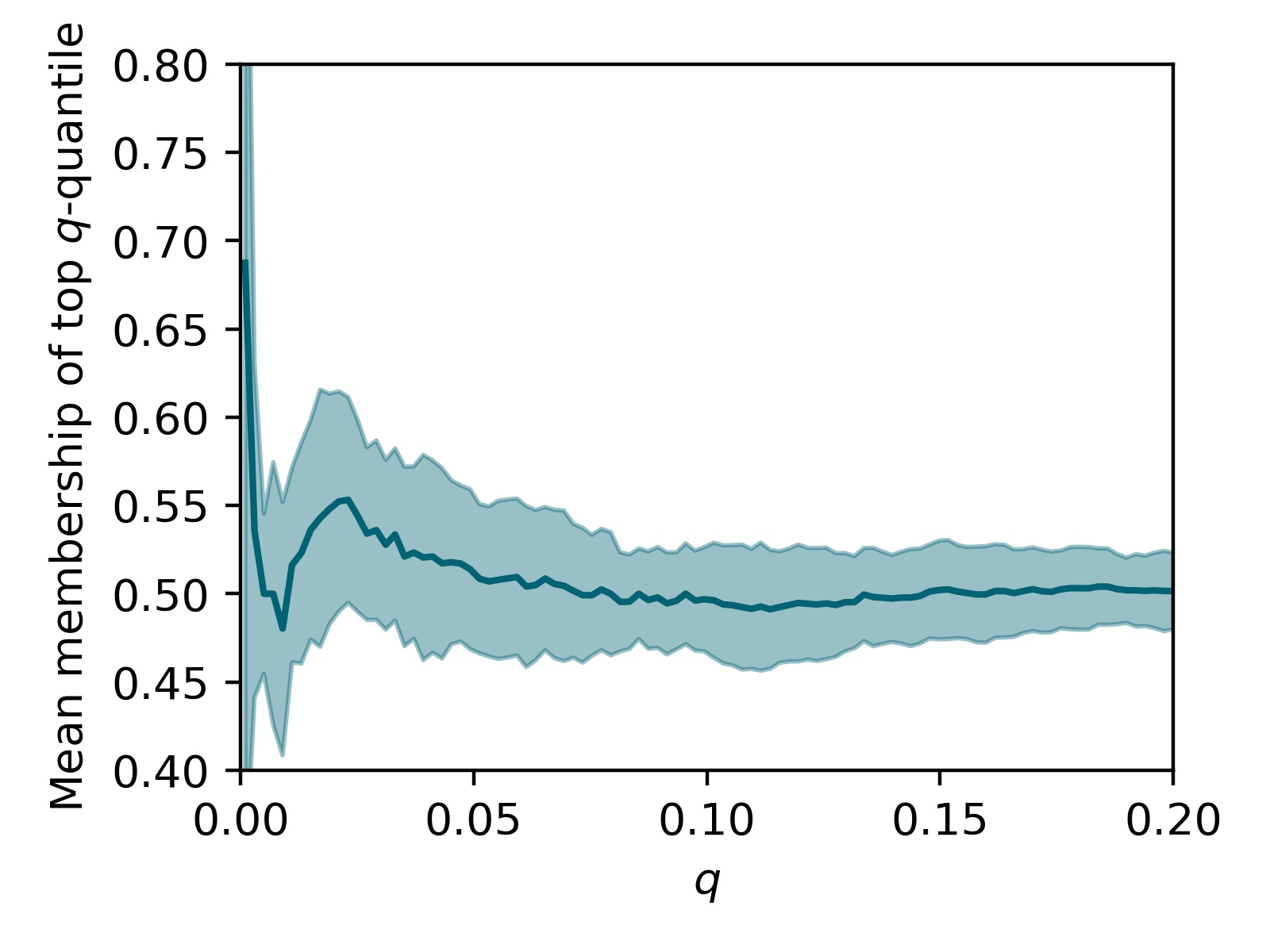}
    \caption{LOGAN 0}
    \end{subfigure}
    \begin{subfigure}{0.33\textwidth}
    \centering
    \includegraphics[width=\textwidth]{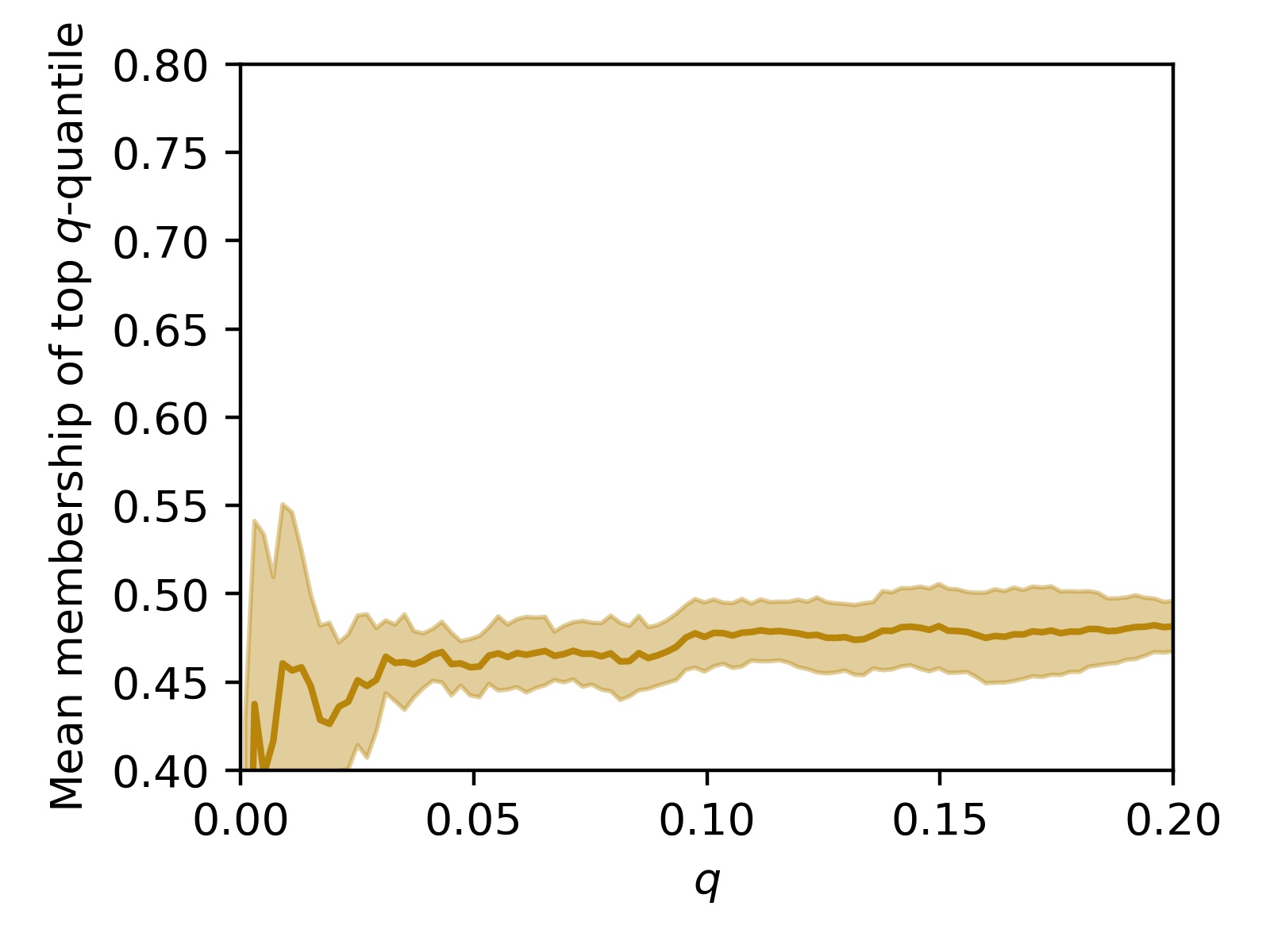}
    \caption{LOGAN D1}
    \end{subfigure} 
    \begin{subfigure}{0.33\textwidth}
    \centering
    \includegraphics[width=\textwidth]{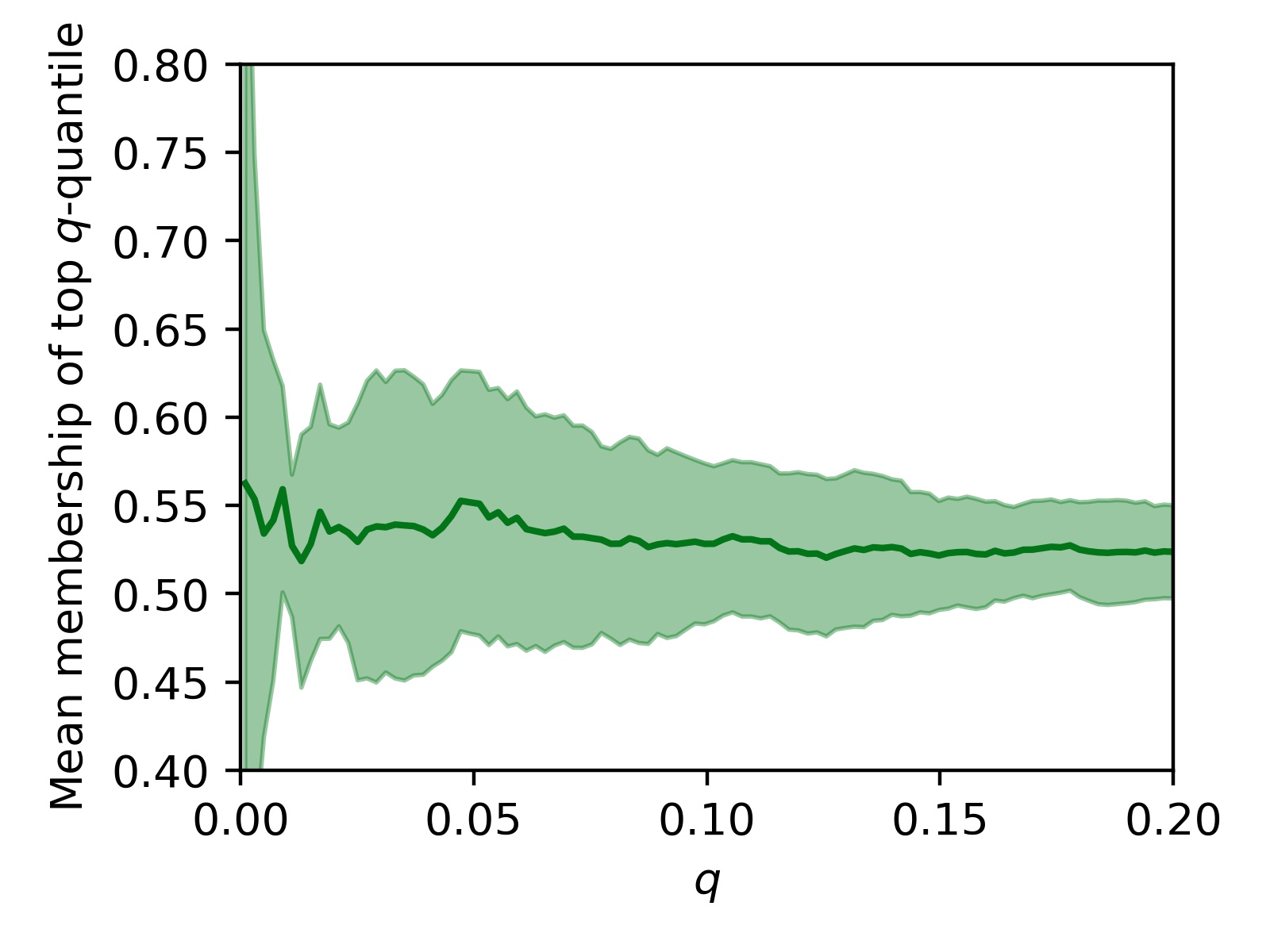}
    \caption{GAN-leaks 0}
    \end{subfigure}
    \begin{subfigure}{0.33\textwidth}
    \centering
    \includegraphics[width=\textwidth]{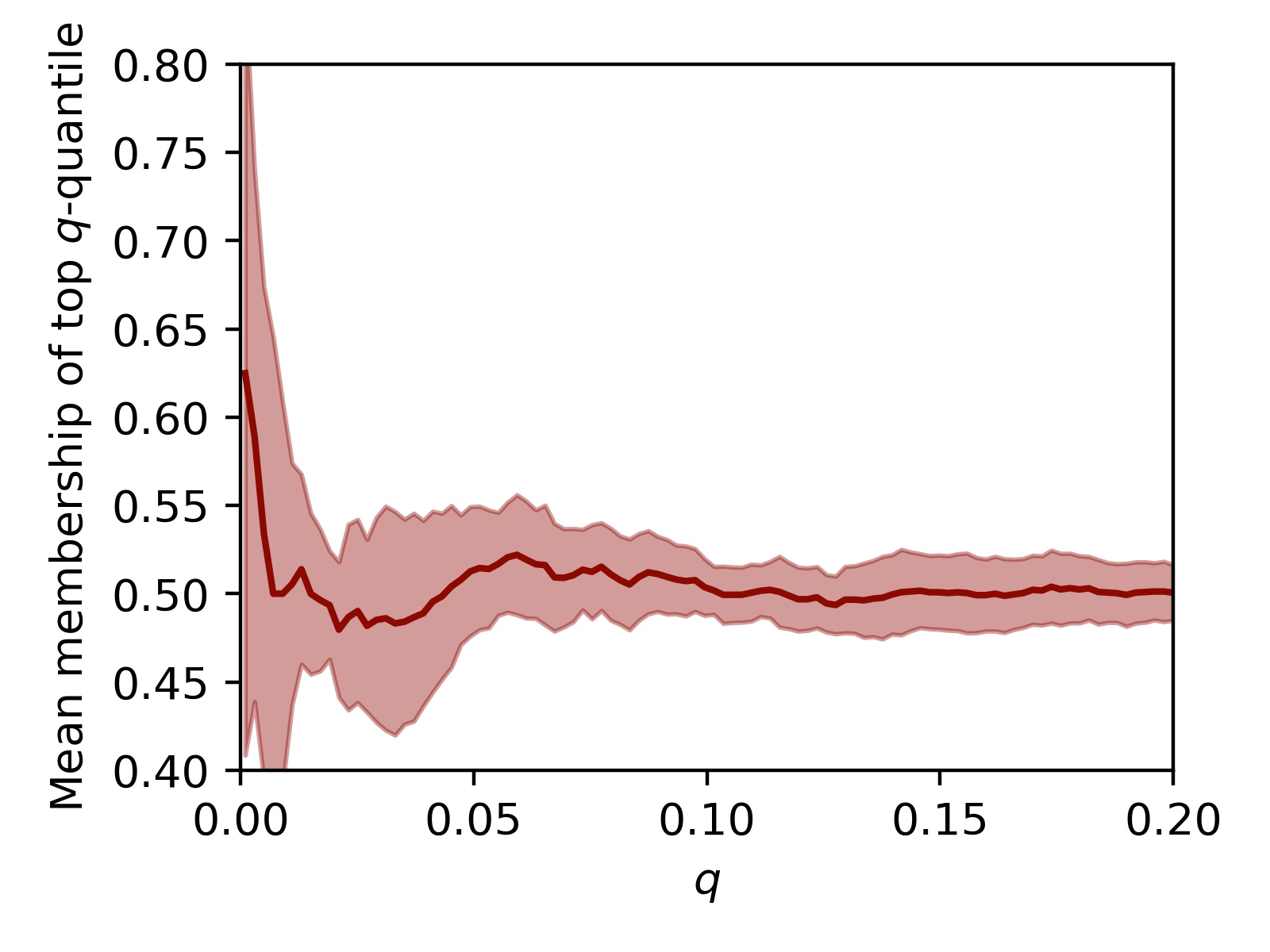}
    \caption{GAN-leaks CAL}
    \end{subfigure}
    \begin{subfigure}{0.33\textwidth}
    \centering
    \includegraphics[width=\textwidth]{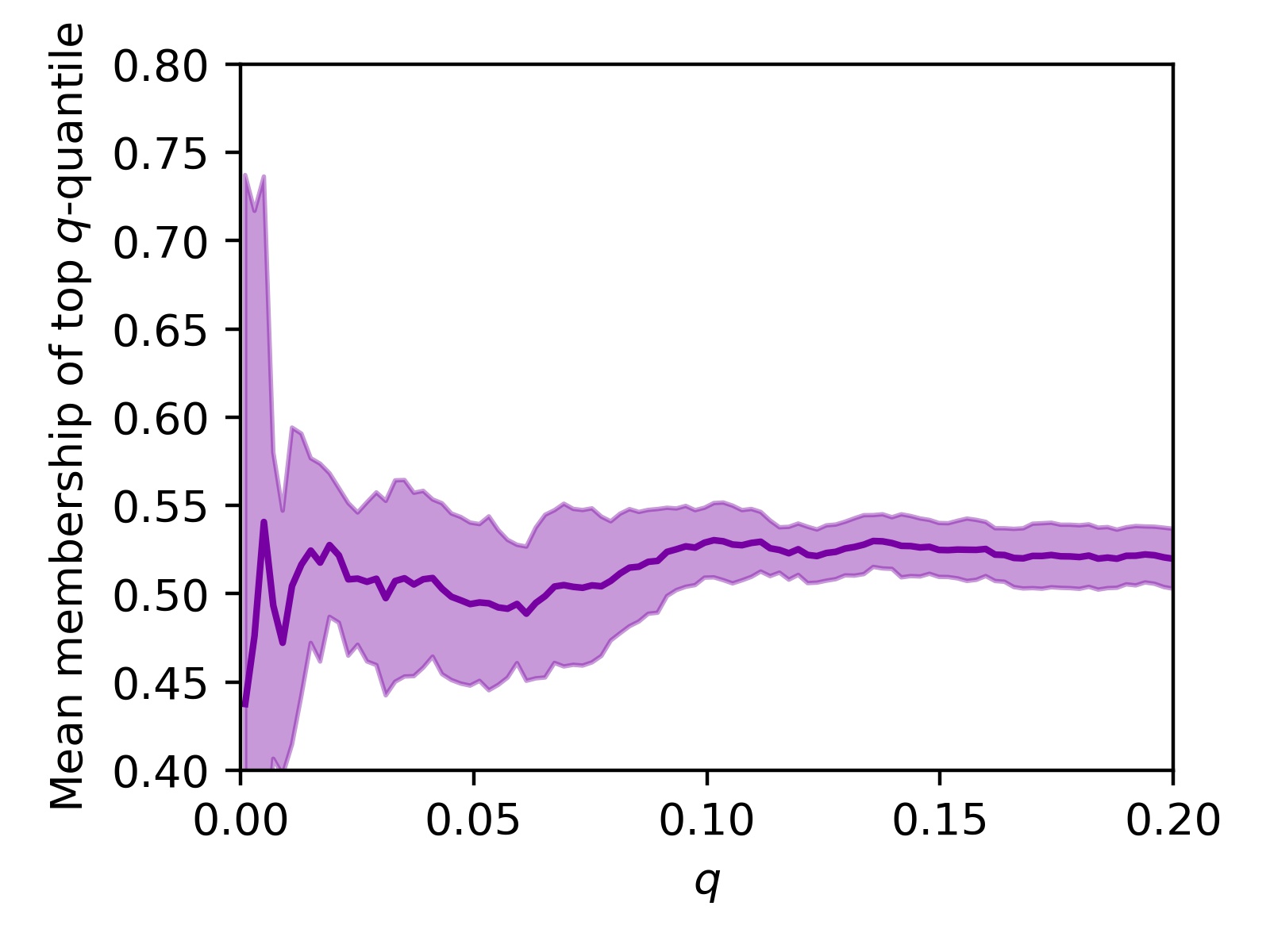}
    \caption{MC}
    \end{subfigure}
    \caption{\emph{DOMIAS is better at high-precision attacks than baselines on CelebA image data.} For example, an attacker could attack only the examples with top 2\% scores, and get a precision of $P=65.7\pm11.6\%$---much higher than the second-best method LOGAN 0, scoring $P=54.8\pm6.5\%$.}
    \label{fig:high-precision-celeba}
\end{figure*}

\section{DISTRIBUTION SHIFT $\cD_{ref}$ AND $\cD_{mem}$} \label{appx:distributional_shift}
There may exist a distributional shift between reference and training data. Because DOMIAS is primarily intended as a tool for data publishers to test their own synthetic data vulnerability, it is recommended that testing is conducted with a reference dataset from the same distribution (e.g. a hold-out set): this effectively tests the worst-case vulnerability. Hence, our work focused on the case where there is no shift.

Nonetheless, reference data may not always come from the same target distribution. For example, reference data may come from a different country, or synthetic data may be created by intentionally changing some part of the real data distribution, e.g. to include fairness guarantees \citep{xu2019achieving,vanBreugel2021DECAF:Networks}. Thus, let us assume there is a shift and that the reference data $\mathcal{D}_{ref}$ comes from $\tilde{p}_R$, a shifted version of $p_R$ (i.e. the distribution from which $\mathcal{D}_{mem}$ is drawn). We give a specific example and run an experiment to explore how this could affect DOMIAS attacking performance.

Let us assume there is a healthcare provider that publishes $\mathcal{D}_{syn}$, a synthetic dataset of patients suffering from diabetes, based on underlying data $\mathcal{D}_{mem}\sim p_R$. Let us assume there is an attacker that has their own data $\mathcal{D}_{ref}\sim \tilde{p}_R$, for which some samples have diabetes ($A=1$), but others do not ($A=0$). We assume that $A$ itself is latent and unobserved (s.t. the attacker cannot just train a classification model) and that there is a shift in the distribution of $A$ (i.e. with a slight abuse of notation $\tilde{p}_R(A=1)<1$). Diabetes is strongly correlated with other features $X$ in the data, additionally we assume the actual condition distribution $p_R(X|A)$ is fixed across datasets. This implies the reference and membership set distributions can be written respectively as:
\begin{align}
    \tilde{p}_R(X) &= \tilde{p}_R(A=1)p(X|A=1) + \tilde{p}_R(A=0)p(X|A=0) \\
    p_R(X) &= p(X|A=1)
\end{align}
Since $p_R(X|A=1)\neq p_R(X|A=0)$ and $\tilde{p}_R(A=1)\neq 1$, there is a distributional shift between $\tilde{p}_R$ and $p_R$.

Now let us see how different attackers perform in this setting as a function of the amount of shift. Evidently, since some of the baselines do not use reference data, some attackers will be unaffected, but we should expect DOMIAS performance to degrade. We take the Heart Failure dataset, which indeed has a feature denoting diabetes,. We vary the amount of shift of $\tilde{p}_R$ w.r.t. $p_R$, from $\tilde{p}(A=0)=0$ (no shift), to $\tilde{p}(A=0)=0.8$ (a large shift and the original Heart Failure non-diabetes prevalence). Let us assume test data follows the attacker's existing dataset, i.e. $\tilde{p}_R$. This gives Figure \ref{fig:distributional_shift}.

\begin{figure*}
    \centering
    \includegraphics[width=0.7\textwidth]{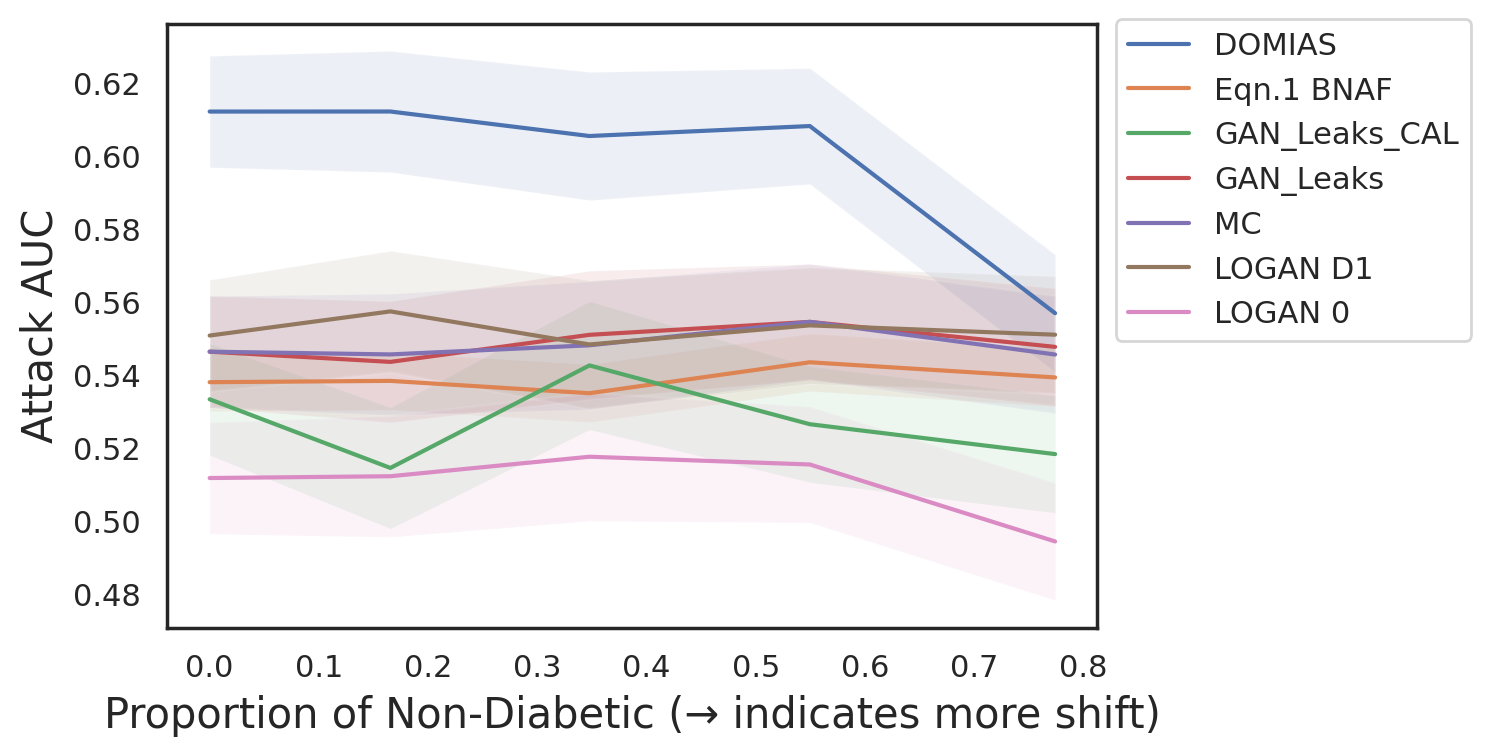}
    \caption{\textit{Effect of distributional shift on DOMIAS performance.} A distributional shift between $\mathcal{D}_{mem}$ and $\mathcal{D}_{ref}$ degrades attacking performance, but preliminary experiments show that for small to moderate shifts it is still preferable to use reference data even though it is slightly shifted.}
    \label{fig:distributional_shift}
\end{figure*}

We see performance of DOMIAS degrades with increasing shift, due to it approximating $p_R$ with $\tilde{p}_R$, affecting its scores (Eq. 2). However, we see that for low amounts of shift this degradation is minimal and we still perform beter than not using the reference dataset (baseline Eq. 1 (BNAF)). This aligns well with the results from 5.2, Figure 4, that showed that an inaccurate approximation of $p_R$ due to few samples is still preferable over not using any reference data.

\end{document}